\documentclass[11pt]{article}

\usepackage[T1]{fontenc}
\usepackage[utf8]{inputenc}
\usepackage[margin=1in]{geometry}
\usepackage{authblk}
\usepackage{amsmath,amssymb,amsthm}
\usepackage{graphicx}
\usepackage{subcaption}
\usepackage{xspace}
\usepackage{algorithm}
\usepackage{algorithmic}
\usepackage{xcolor}
\usepackage[numbers,sort&compress]{natbib}
\usepackage[hidelinks]{hyperref}

\theoremstyle{plain}
\newtheorem{theorem}{Theorem}[section]
\newtheorem{proposition}[theorem]{Proposition}
\newtheorem{corollary}[theorem]{Corollary}
\newtheorem{lemma}{Lemma}
\theoremstyle{definition}
\newtheorem{definition}{Definition}
\theoremstyle{remark}

\newcommand\methodname{HAAM\xspace}
\newcommand{\norm}[1]{\left\lVert#1\right\rVert}

\title{Modeling Heterophily in Multiplex Graphs: An Adaptive Approach for Node Classification}

\author[1]{Kamel Abdous\thanks{Kamel Abdous and Nairouz Mrabah contributed equally to this work.}}
\author[1]{Nairouz Mrabah\protect\footnotemark[1]}
\author[1]{Mohamed Bouguessa}
\affil[1]{Department of Computer Science, University of Quebec at Montreal, Montreal, QC, Canada\\
\texttt{abous.kamel@courrier.uqam.ca}; \texttt{mrabah.nairouz@etsmtl.livia.ca}; \texttt{bouguessa.mohamed@uqam.ca}}
\date{}

\begin{document}
\maketitle

\begin{abstract}
Existing multiplex graph models often assume homophily, where connected nodes tend to belong to the same class or share similar attributes. Consequently, these models may struggle with graphs exhibiting heterophily, where connected nodes typically belong to different classes and have dissimilar attributes. While recent methods have been developed to learn reliable node representations from unidimensional graphs with heterophily, they do not fully address the complexities of multiplex graphs. In a multiplex graph, nodes are linked through multiple types of edges (referred to as dimensions), which can simultaneously exhibit homophilic and heterophilic interactions. To address this gap, we propose \methodname, a novel method for node classification in multiplex graphs that adapts to both homophilic and heterophilic dimensions. \methodname introduces dimension-specific compatibility matrices to model varying degrees of homophily and heterophily across dimensions. A key innovation is its use of a product of trainable low-pass and high-pass filters, approximated via Chebyshev polynomials, to capture both smooth and abrupt changes in the graph signal. By composing these filters and optimizing label predictions using a proximal-gradient method, \methodname dynamically adjusts to the heterophilic characteristics of each dimension. Extensive experiments on synthetic and real-world datasets provide evidence that \methodname captures the complex interplay of homophilic and heterophilic interactions in multiplex graphs, and tends to yield improved node classification performance compared to state-of-the-art methods.
\end{abstract}


\bigskip
\section{Introduction}
Multiplex graphs offer a comprehensive framework for modeling interconnected systems by capturing multiple layers of interactions within a single structure \cite{battiston2017new}. Unlike traditional unidimensional graphs \cite{mrabah2022rethinking}, \cite{mrabah2024contrastive}, \cite{mrabah2022escaping}, which represent relationships with a single type of edge, multiplex graphs enable simultaneous and diverse connections between the same set of nodes \cite{abdous2024geometric}. Each type of connection represents a distinct graph dimension. This capability is particularly helpful in modeling complex systems where entities interact through various channels, such as social networks \cite{berlingerio2013multidimensional}, biological networks \cite{oughtred2021biogrid}, and online recommendation systems \cite{sun2019multi}. By preserving the distinct nature of each interaction, multiplex graphs support a more accurate and nuanced analysis of system dynamics, which may lead to deeper insights into how different layers of connectivity contribute to the overall behavior \cite{hmge2023}.

Graph modeling techniques often rely on the homophily assumption, which presumes that nodes are more likely to connect if they share the same class or closely related attributes \cite{10597953}. While this assumption underpins many traditional methods, its applicability can be limited in systems where connections arise from diverse or even opposing factors \cite{zhu2020beyond}. In such cases, heterophily plays an important role, as nodes with differing attributes may be more likely to connect \cite{luan2022revisiting}. For instance, professional networks foster collaborations between individuals with complementary skills \cite{barranco2019heterophily}, while biological systems often rely on interactions between entities that assume diverse yet interdependent roles \cite{lim2021large}.

Despite progress in modeling homophily in multiplex graphs \cite{WANG2023120552}, \cite{jing2022x}, \cite{jing2021hdmi}, \cite{mitra2021semi}, \cite{pio2021multiverse}, \cite{sadikaj2023semi}, the role of heterophily remains relatively underexplored. Although recent efforts have investigated heterophily in traditional unidimensional graphs \cite{LIU2025125274}, \cite{JIN2024124460},  \cite{pan2023beyond}, \cite{zheng2022graph}, extending this framework to multiplex graphs introduces additional challenges. Multiplex systems often exhibit conflicting patterns of connectivity, where nodes may be dissimilar in one dimension but similar in another \cite{boutemine2017mining}. Furthermore, certain dimensions may emphasize homophilic interactions (e.g., individuals with shared interests), while others reflect heterophilic dynamics (e.g., collaborations between individuals with different expertise). Addressing these challenges is important for advancing tasks like node classification, where leveraging both homophilic and heterophilic connections can improve predictive accuracy.

To address these challenges, we introduce \methodname (Heterophily-Aware Adaptive Multiplex model), a novel framework tailored to node classification in multiplex graphs\footnote{This work focuses on multiplex graphs with homogeneous nodes, where nodes of the same type are connected through multiple dimensions, each representing a distinct type of relation. While multiplex graphs involve multiple types of edges, this setting is distinct from heterogeneous graphs, which involve both multi-typed nodes and edges. Other graph types, such as heterogeneous or dynamic/time-evolving graphs, fall outside the scope of this study and warrant further investigation.}. Unlike existing methods that primarily focus on homophilic structures, \methodname explicitly accommodates both homophilic and heterophilic interactions across graph dimensions. More precisely, \methodname uses learnable and dimension-specific compatibility matrices \cite{zhu2021graph} to capture the varying levels of homophily and heterophily across dimensions. Additionally, we propose a combination between learnable low-pass and high-pass spectral Chebyshev filters \cite{defferrard2016convolutional} to extract smooth (i.e., low-frequency homophilic) and rapidly-changing (i.e., high-frequency heterophilic) information from node interactions. In particular, our framework leverages a mathematically derived method that applies two filters sequentially via the product of their Chebyshev interpolation in the spectral domain. Finally, the proposed model is trained using two loss functions: the traditional cross-entropy loss, and a second loss that minimizes the divergence between dimension-specific and consensus predictions while promoting sparsity in the consensus predictions. To handle the non-smooth regularization that induces sparsity, we adopt a proximal-gradient optimization framework.

Spectral polynomial filters have been shown to approximate a wide range of spectral filters effectively \cite{wang2022powerful}, making them suitable in spectral graph convolutions for both homophilic and heterophilic graphs. In particular, Chebyshev polynomials \cite{defferrard2016convolutional}, \cite{he2022convolutional} are recognized for their strong approximation capabilities thanks to the properties of the Chebyshev basis and the capacity to minimize the Runge phenomenon. Unlike GREET \cite{liu2023beyond}, PolyGCL \cite{chen2024polygcl}, and TFE-GNN \cite{duan2024unifying}, which all use linear combinations or concatenations of Chebyshev filters, our approach employs the product of low-pass and high-pass filters. As an advantage, the product of the filters can capture non-linear interactions between the low-frequency and high-frequency components. Furthermore, the product of filters introduces higher-order Chebyshev terms (up to $2K$ for two Chebyshev filters of order $K$), which enables capturing more complex interactions between low and high frequencies.

\textbf{Contributions:} 
\begin{itemize}
    \item \textbf{Heterophily-aware adaptive multiplex model (\textsc{\methodname}).}
    We introduce a principled framework for multiplex node classification that explicitly models relation-specific level of homophily/heterophily via learnable compatibility matrices and reconciles dimension-wise predictions into a \emph{sparse} consensus label distribution through a proximal-gradient formulation (Alg.~\ref{algo:algo}).

    \item \textbf{Product-composed Chebyshev spectral filtering.}
        We propose to fuse learnable low-pass and high-pass Chebyshev filters by composition (matrix product) rather than by linear mixtures, enabling non-linear low/high-frequency interactions and implicitly introducing higher-order terms up to degree $2K$ (Sec.~\ref{sec:product_filter}). We characterize the composed operator in the spectral domain (Proposition~\ref{prop.A} and Corollary~\ref{coro.1}) and derive a Chebyshev product-to-sum expansion (Proposition~\ref{prop:2}) that avoids explicit $N\times N$ matrix products and yields an efficient implementation.

    \item \textbf{Theoretical guarantees for stability and generalization.}
    We establish bounded input, bounded output stability bounds for Chebyshev bases and polynomial filters and extend them to the product-composed operator (Propositions~\ref{prop:cheb_basis_bound}--\ref{prop:bibo_cheb} and Corollary~\ref{coro:product_stability}). Furthermore, we prove that the row-wise softmax is Lipschitz-stable (Proposition~\ref{prop:softmax_stability}) and formalize a bandwise noise attenuation effect induced by the product fusion (Proposition~\ref{prop:bandwise_attenuation}). Finally, we provide a statistical generalization bound linking the generalization gap to the operator norms $\|\hat{L}_d\|_2\|H_d\|_2$ and to the learned Chebyshev coefficients via Corollary~\ref{coro:product_stability} (Proposition~\ref{prop:rad_scaling} and Theorem~\ref{thm:generalization}; Sec.~\ref{sec:generalization}).

    \item \textbf{Comprehensive empirical validation across homophily regimes.}
    We conduct extensive experiments on synthetic multiplex graphs with controlled homophily ratios and on real-world multiplex datasets. Results show that \methodname\ consistently outperforms competitive multiplex baselines and adapted heterophily-aware unidimensional methods (Fig.~\ref{fig:synthetic_results} and Table~\ref{table:node_classification}). Ablations and sensitivity analyses further validate the impact of compatibility modeling, proximal consensus, and product-based filtering (Table~\ref{table:ablation_study} and Fig.~\ref{fig:sensitivity_analysis}).
\end{itemize}

\section{Related Work}

This section reviews existing approaches for modeling multiplex graphs and the advancements in addressing heterophily within unidimensional graphs.

\subsection{Models for Multiplex Graphs}

Many methods designed for multiplex graphs have traditionally focused on capturing homophily patterns. DMGI \cite{park2020dmgi} integrates embeddings from various types of node relations using a consensus regularization framework, a bilinear discriminator, along with a mutual information maximization mechanism \cite{hjelm2018learning} that pulls similar nodes together and dissimilar ones apart. HDMI \cite{jing2021hdmi} builds upon the contrastive loss of DMGI by including a higher-order objective function that incorporates node features. X-GOAL \cite{jing2022x} further improves the contrastive loss by pairing topologically similar nodes and nodes within the same cluster as positive and negative pairs. SSDCM \cite{mitra2021semi} introduces a self-supervised framework for learning node representations in multiplex graphs. The model leverages a cluster-based graph summary that improves the discriminative power of node embeddings. Since all these methods rely on optimizing contrastive loss functions, they inherently group nodes that are topologically close, operating under the assumption of homophily. Beyond fixed-structure contrastive objectives, InfoMGF \cite{shen2024beyond} tackles multiplex graph reliability by refining each graph view to remove task-irrelevant noise and learning a fused graph by maximizing both view-shared and view-unique task-relevant information.

GATNE \cite{cen2019representation} decomposes the node embeddings into base, edge, and attribute embeddings. Neighborhood information is integrated into edge embeddings using a self-attention mechanism. Similarly, mGCN \cite{ma2019multi} uses two types of node embeddings: one that captures interactions within and across dimensions, and another that considers the node embeddings with respect to the entire graph. SSAMN \cite{sadikaj2023semi} combines node embeddings and class label embeddings learned in a semi-supervised setting using a small subset of labeled nodes. HMGE \cite{hmge2023} embeds high-dimensional multiplex graphs by hierarchically encoding the graph dimensions. The method progressively builds hidden graph dimensions that can capture new types of interactions through nonlinear combinations of the original graph structures. DMG \cite{mo2023disentangled} focuses on capturing the common and complementary information across the graph dimensions. The approach takes advantage of disentangled representations to distinguish between shared and unique information. In the same context, MGDCR \cite{mgdcr} minimizes the correlation between inter-dimension and intra-dimension codes.

While recent studies on multiplex graphs have made significant progress, the treatment of heterophily, where nodes with different attributes or classes are more likely to connect, remains only partially explored. This aspect may limit the ability to fully capture the structural diversity observed in complex systems. This work aims to complement recent efforts by proposing a novel approach that explicitly integrates both homophilic and heterophilic interactions across multiplex dimensions in a principled way. The proposed method can accurately capture the full spectrum of interactions within multiplex graph dimensions, including heterophilic and homophilic interactions.

\subsection{Heterophily in Unidimensional Graphs}

In recent years, heterophily has emerged as a significant challenge in modeling unidimensional graphs, leading to the development of various methods to address it. GPR-GNN \cite{chien2020adaptive} is a flexible graph neural network that adapts to homophilic and heterophilic graphs. By learning generalized PageRank weights, GPR-GNN optimizes the balance between node features and topological information, avoiding over-smoothing. LINKX \cite{lim2021large} introduces a simple and scalable method to learn reliable representations on heterophilous graphs. The method separately embeds node features and edges with MLPs and combines the embeddings by concatenation. DGCN \cite{pan2023beyond} follows a similar approach, employing a mixed filter to balance low- and high-frequency information. CPGNN \cite{zhu2021graph} incorporates a learnable compatibility matrix that models the likelihood of connections between nodes of different classes. SELENE \cite{zhong2022unsupervised} proposes a dual-channel embedding pipeline that discriminates between r-ego networks, taking advantage of the attributes of the nodes and the structural information separately. GREET \cite{liu2023beyond} employs an edge discriminator that separates homophilic from heterophilic edges using features and structural information. This is coupled with a dual-channel encoder that processes edges independently with a concatenation of low-pass and high-pass filters. PolyGCL \cite{chen2024polygcl} applies contrastive learning with polynomial filters, making it adaptable to homophilic and heterophilic graphs. Moreover, it introduces a dual-channel filtering mechanism with a linear combination of low-pass and high-pass filters. Similarly, TFE-GNN \cite{duan2024unifying} employs triple filter ensembles and combines low-pass and high-pass filters using linear summations and concatenations.

Although these methods have made significant contributions to addressing heterophily in unidimensional graphs, they fall short when it comes to handling the complexity of multiplex graphs. Multiplex graphs, characterized by multiple dimensions of diverse interactions between nodes, present unique challenges that unidimensional approaches can not tackle. In this paper, we propose a principled approach specifically designed to handle the varying levels of homophily and heterophily across the dimensions of multiplex graphs.

\section{Definitions \& Notations}

Before describing the proposed approach, let us first present the main definitions and notations used throughout the paper. For convenience, Table \ref{tab:notation_summary} in \ref{sec:notation_summary} summarizes the main notation used throughout the paper.

\subsection{Multiplex Graphs} 
We consider a $D$-dimensional multiplex graph $G$, defined as a set of $D$ graphs $G = \left \{G_1, \dots, G_D \right\}$. Each graph $G_d = (V, \, A_d)$, for $d \in \{1, \dots, D\}$, consists of the same set of $N$ nodes $V = \left \{ v_1, \dots, v_N \right\}$ and an adjacency matrix $A_d \in \mathbb{R}^{N \times N}$, where $(A_d)_{ij} = 1$ if there is an edge between nodes $v_i$ and $v_j$ in dimension $d$, and $0$ otherwise. The degree matrix $\Delta_d$ of each adjacency matrix $A_d$ is diagonal, with $(\Delta_d)_{ii} = \sum_{j=1}^N (A_d)_{ij}$, and the associated normalized Laplacian matrix is $L_d = I - \Delta_d^{-\frac{1}{2}} \, \cdot  \, (A_d + I) \, \cdot  \, \Delta_d^{-\frac{1}{2}}$. We define $X \in \mathbb{R}^{N \times F}$ as the node feature matrix, where the $i$-th row corresponds to the feature vector of node $v_i$ and $F$ is the number of features per node. The label matrix is denoted by $Y \in \{0, 1\}^{N \times C}$, where each row corresponds to a node, and each node is assigned exactly one of the $C$ labels.

\subsection{Homophily Ratio} Let $\mathcal{Y}_d \in \mathbb{N}^{C \times C}$ be the matrix that measures class-wise connectivity, where each entry $(\mathcal{Y}_{d})_{ij}$ indicates the number of edges connecting nodes of class $i$ to nodes of class $j$ in dimension $d$. The homophily ratio for dimension $d$ is defined in Eq. (\ref{eq:homophily_ratio}). The proportion of heterophilic edges is given by $1 - h_d$. As $h_d$ decreases, the level of homophily in dimension $d$ decreases, while the level of heterophily increases.
\begin{equation}
    h_d = \frac{\sum_{i=1}^C \left(\mathcal{Y}_{d}\right)_{ii}}{\sum_{i,j=1}^{C} \left(\mathcal{Y}_{d}\right)_{ij}}.
\label{eq:homophily_ratio}
\end{equation}

\subsection{Chebyshev Polynomials} For $x \in \mathbb{R}$, the Chebyshev polynomials $T_k(x)$ are recursively defined by the relation $T_{k+1}(x) = 2 \, x \, T_k(x) - T_{k-1}(x)$, with initial conditions $T_0(x) = 1$ and $T_1(x) = x$. For matrices, this recurrence is applied element-wise, making Chebyshev polynomials useful in approximating spectral functions.

\subsection{Spectral Graph Convolution} Spectral graph neural networks operate based on spectral graph convolutions. Recall that $X \in \mathbb{R}^{N \times F}$ denotes the node feature matrix introduced in Sec.~3.1, where each row corresponds to one node and each column to one feature. For dimension $d$, this operation is given by:
\begin{equation}
    Z_d = f_d(L_d) \, X, \qquad \forall d \in \left\{1, \dots, D \right\},
\end{equation}
\noindent where $Z_d \in \mathbb{R}^{N \times F'}$ is the transformed feature matrix for dimension $d$ (with $F'$ the embedding dimension), and $f_d(L_d) \in \mathbb{R}^{N \times N}$ is the spectral filter computed using the normalized Laplacian matrix $L_d$. Let $U_d \, \Lambda_d \, U_d^\top$ be the eigendecomposition of $L_d$, $\Lambda_d = \text{diag}(\lambda_{d}^{(1)}, \dots, \lambda_{d}^{(N)})$ is the matrix of eigenvalues representing the graph frequencies, and $U_d = [u_{d}^{(1)}, \dots, u_{d}^{(N)}]$ contains the eigenvectors. The graph Fourier transform of a graph signal $x \in \mathbb{R}^{N}$ is defined as $\hat{x} = U_d^\top \, x$, and the inverse transform is $x = U_d \, \hat{x}$. In the spectral domain, the spectral filter modulates the frequency response as expressed in:
\begin{equation}
    Z_d = f_d(L_d) \, X = U_d \, f_d(\Lambda_d) \, U_d^\top \, X.
\end{equation}

\section{The Proposed \methodname Approach}

In this section, we describe the proposed approach \methodname (Heterophily-Aware Adaptive Multiplex model).

\subsection{Overall Framework}

\begin{figure*}
  \centering
  \includegraphics [width=\linewidth]{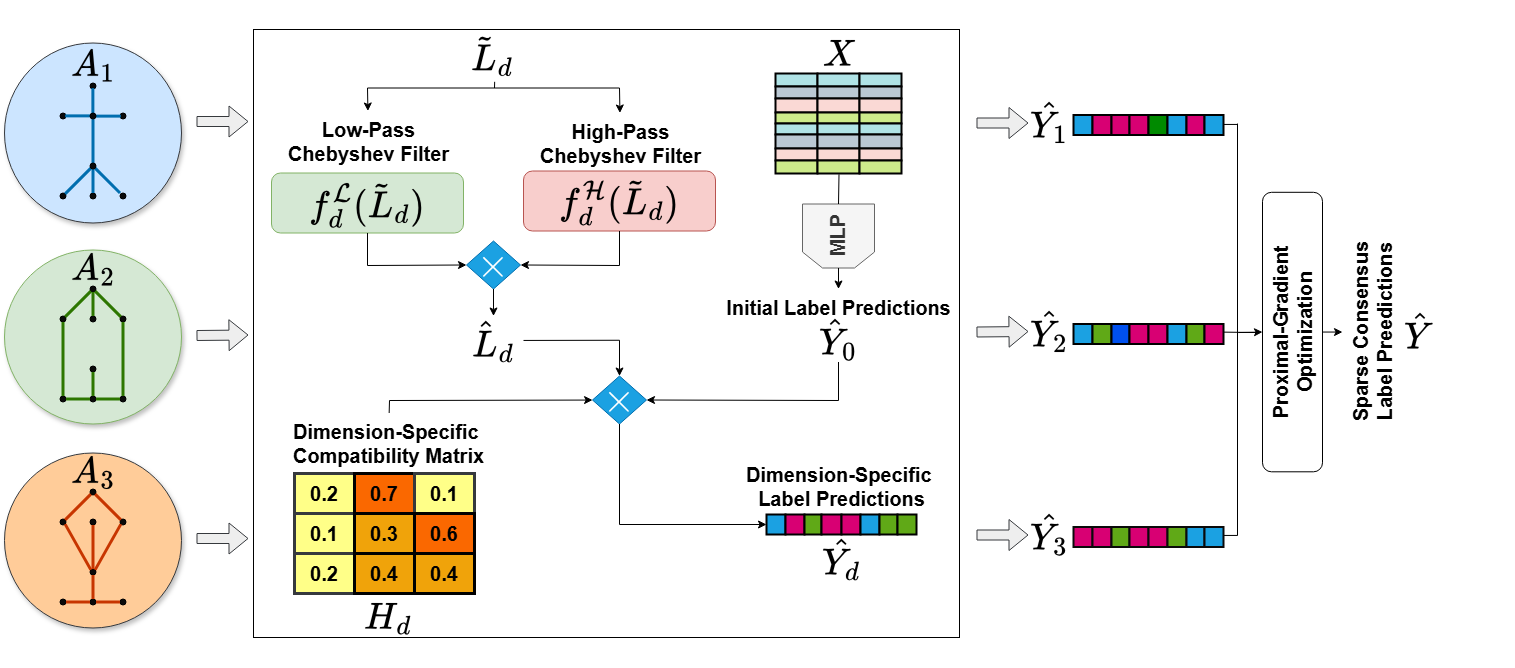}
  \caption{The architecture of \methodname.}
  \label{fig:overall_framework}
\end{figure*}

Fig. \ref{fig:overall_framework} depicts the architecture of \methodname. The process starts by embedding the features $X$ using a Multilayer Perceptron (MLP) to generate an initial distribution of label predictions $\hat{Y}_0 \in \mathbb{R}^{N \times C}$. In high heterophily settings, the graph topology reflects more complex relations between the label distribution and node features compared to homophilous graphs \cite{lim2021large}. To overcome this, we embed $X$ independently from the graph structures to extract a prior distribution of node labels, which is then refined in the subsequent stages. 

We formulate learnable low-pass and high-pass Chebyshev filters for each dimension $d$, denoted by $f_{d}^{\mathcal{L}}$ and $f_{d}^{\mathcal{H}}$, respectively. The filters are computed based on the rescaled Laplacian matrices $\tilde{L}_d = 2 \, L_{d} \, / \, \lambda_{d}^{\text{max}} - I$, where $\lambda_{d}^{\text{max}}$ is the greatest eigenvalue of the normalized Laplacian $L_{d}$. The low-pass filters smooth neighboring signals, emphasizing homophilic properties in the graph structures. Conversely, the high-pass filters capture significant variations between adjacent nodes, highlighting heterophilous properties. 

For each dimension, the low-pass and high-pass spectral filters extract distinct information, which we combine through a matrix product to obtain
\begin{equation}
    \hat{L}_d = f_{d}^{\mathcal{L}}(\tilde{L}_d) \cdot f_{d}^{\mathcal{H}}(\tilde{L}_d), \qquad \hat{L}_d \in \mathbb{R}^{N \times N}.
\end{equation}
This contrasts with previous work that employs linear combinations and concatenations to fuse filters \cite{liu2023beyond, chen2024polygcl, duan2024unifying}. 

The next step is to compute the updated label predictions $\hat{Y}_d$ for each dimension $d$. To this end, we model the probability of connections between nodes in different classes using an empirically estimated compatibility matrix $H_d \in \mathbb{R}^{C \times C}$ \cite{zhu2021graph}. As explained in Sec. \ref{sec:compa_matrix}, the entry $(H_d)_{c_1 c_2}$ represents the likelihood that nodes from class $c_1$ in dimension $d$ connect to nodes from class $c_2$. 
Both matrices, $\hat{L}_d$ and $H_d$, capture heterophily and homophily within the graph. While $\hat{L}_d$ diffuses information across the topology, $H_d$ adjusts the predictions based on inter-class connection probabilities. Accordingly, the updated predictions $\hat{Y}_d$ are expressed as:
\begin{equation} \label{eq:gnn_update_rule}
    \hat{Y}_d = \mathrm{softmax}(\hat{L}_d \, \cdot \, \hat{Y}_0 \, \cdot \, H_d), 
\end{equation}
where $\hat{Y}_d \in \mathbb{R}^{N \times C}$ denotes the dimension-specific label predictions and $H_d \in \mathbb{R}^{C \times C}$ is the compatibility matrix introduced in Sec.~\ref{sec:compa_matrix}.

Finally, we generate the consensus label predictions $\hat{Y}$ from all the dimension-specific predictions $\hat{Y}_d$. Specifically, we minimize the divergence between $\hat{Y}_d$ and $\hat{Y}$ while maximizing sparsity in the consensus predictions. To manage the non-smooth regularization that induces sparsity, we utilize proximal-gradient optimization.

\subsection{Composition of Chebyshev Filters}\label{sec:product_filter}

We define the filters $f^{\mathcal{L}}_{d}$ and $f^{\mathcal{H}}_{d}$ as Chebyshev polynomials \cite{defferrard2016convolutional} to approximate optimal spectral filters.  
Let $K$ denote the degree of the polynomial filters, then:
\begin{equation}\label{eq:low_high_polynomials}
    f_d^{\mathcal{L}}(\tilde{L}_d) = \sum_{k=0}^K \theta_k^{\mathcal{L}, d} \, T_k(\tilde{L}_d), \;\; f_d^{\mathcal{H}}(\tilde{L}_d) = \sum_{k=0}^K \theta_{k}^{\mathcal{H}, d} \, T_k(\tilde{L}_d),
\end{equation}
where $\theta_k^{\mathcal{L}, d}$ and $\theta_k^{\mathcal{H}, d}$ are the polynomial coefficients of the low-pass and high-pass filters, respectively, for dimension $d$. 
These coefficients control which spectral components of the graph signal are emphasized. Low-pass filters aim to retain low-frequency components, which correspond to smooth variations across the graph. These components are represented by lower eigenvalues. On the other hand, high-pass filters are designed to capture high-frequency components, which correspond to rapid variations between neighboring nodes. These rapid variations are captured by higher eigenvalues. Intuitively, low-pass filters capture homophilic relations, while high-pass filters capture heterophilic relations.

Similar to \cite{defferrard2016convolutional}, it is possible to optimize the coefficients $\theta_k^{\mathcal{L}, d}$ and $\theta_k^{\mathcal{H}, d}$ via gradient descent.
However, the unconstrained coefficients might not capture the expressive power of Chebyshev coefficients and can lead to overfitting \cite{he2022convolutional}. We aim to approximate arbitrary low-pass and high-pass filters that could adapt to the specific characteristics of the multiplex graph. To achieve this, we adopt a double reparametrization approach. Given $\mathcal{F} \in \left\{\mathcal{L}, \mathcal{H} \right \}$, we first reparameterize 
$\theta_k^{\mathcal{F},d}$ by a vector $\gamma^{\mathcal{F}, d} \in \mathbb{R}^{K+1}$ to capture the characteristics of Chebyshev coefficients, following the formulation in \cite{he2022convolutional}:
\begin{equation}\label{eq:theta_reparam}
    \theta_k^{\mathcal{F}, d} = \frac{2}{K + 1} \sum_{j=0}^K \gamma_j^{\mathcal{F}, d} \:\: T_k\left(cos\left(\frac{j + 1/2}{K+1} \pi\right)\right).
\end{equation}
Constraining the coefficients allows to polynomially approximate an arbitrary spectral filter with an optimal convergence rate \cite{he2022convolutional}. Second, we further reparametrize $\gamma^{\mathcal{F}, d}$ using prefix difference and prefix sum \cite{chen2024polygcl}:
\begin{equation}\label{eq:gamma_high_low}
    \gamma_i^{\mathcal{L}, d} = \gamma_0^d - \sum_{j=1}^i \gamma_j^d, \ \ \ \ \gamma_i^{\mathcal{H}, d} = \sum_{j=0}^i\gamma_j^d,
\end{equation}
where $\gamma_0^{\mathcal{F}, d} = \gamma_0^d$ is a predefined initial value and $\gamma^d = \left[ \gamma^d_1, \dots,  \gamma^d_K \right]$ is a vector of non-negative learnable parameters. By substituting $\gamma_i^{\mathcal{L}, d}$ and $\gamma_i^{\mathcal{H}, d}$ into Eq. (\ref{eq:theta_reparam}), we derive the low-pass and high-pass filters, respectively. This process ensures that the Chebyshev polynomials $f_d^{\mathcal{L}}$ and $f_d^{\mathcal{H}}$ effectively approximate filters with low-pass and high-pass properties.

The filters $f_d^{\mathcal{L}}(\tilde{L}_d)$ and $f_d^{\mathcal{H}}(\tilde{L}_d)$ capture two distinct perspectives for each dimension: one focusing on homophilic relations, the other on heterophily. 
Since we aim to elaborate an adaptive approach, it is essential to combine these filters. Previous work \cite{liu2023beyond, chen2024polygcl, duan2024unifying} has explored linear combinations and concatenations to fuse filters. In this work, we propose an alternative method by utilizing the product of $f_d^{\mathcal{L}}(\tilde{L}_d)$ and $f_d^{\mathcal{H}}(\tilde{L}_d)$. The product is derived from the composition of the two filters as explained in the Prop. \ref{prop.A}.

\begin{proposition} [Spectral response of filter composition. Proof in \ref{proof:prop1}] \label{prop.A}
The application of a low-pass filter $f^{\mathcal{L}}(L)$ followed by a high-pass filter $f^{\mathcal{H}}(L)$ to a graph signal $x$ is equivalent to the application of a filter $f(L)$ whose eigenvalues are equal to the element-wise product of the eigenvalues of $f^{\mathcal{L}}(L)$ and $f^{\mathcal{H}}(L)$. Formally, the filter output is given by:
\[
y = f(L) \, x = U \left( f^{\mathcal{H}}(\Lambda) \cdot f^{\mathcal{L}}(\Lambda) \right) U^{\top} x,
\]
where $y$ is the filter output, $\Lambda$ is the eigenvalue matrix of $L$, and $U$ contains the corresponding eigenvectors.
\end{proposition}

\begin{corollary} [Order-invariance of filter composition. Proof in \ref{proof:coro1}] \label{coro.1}
The composition of a low-pass filter $f^{\mathcal{L}}(L)$ followed by a high-pass filter $f^{\mathcal{H}}(L)$ is order-invariant as expressed in:
\[
f^{\mathcal{H}}(L) \, \cdot \, f^{\mathcal{L}}(L) = f^{\mathcal{L}}(L) \, \cdot \, f^{\mathcal{H}}(L).
\]
\end{corollary}

The Corollary implies that the output of the filter is the same regardless of whether the high-pass or low-pass filter is applied first, as the operation is commutative. 
Compared with the linear combination, the product can capture non-linear interactions between the low-frequency and high-frequency components and introduces higher-order Chebyshev terms, which allow for capturing more complex interactions. However, computing the product of $N \times N$ matrices is computationally expensive. Moreover, applying gradient descent and backpropagation through the products of large matrices incurs substantial memory costs. To mitigate this issue, we express the product of Chebyshev polynomials as a linear combination of Chebyshev polynomials weighted by the coefficients $\theta_i^{\mathcal{L}, d} \, \theta_i^{\mathcal{H}, d}$, as explained in Prop. \ref{prop:2}.

\begin{proposition} [Chebyshev product-to-sum expansion. Proof in \ref{proof:prop2}] \label{prop:2}
Given the low-pass Chebyshev filter $f_d^{\mathcal{L}}(\tilde{L}_d)$ and the high-pass Chebyshev filter $f_d^{\mathcal{H}}(\tilde{L}_d)$, the product of these filters $\hat{L}_d = f_d^{\mathcal{L}}(\tilde{L}_d) \cdot f_d^{\mathcal{H}}(\tilde{L}_d)$ can be expressed as a sum of Chebyshev polynomials weighted by $\theta_i^{\mathcal{L}, d} \: \theta_j^{\mathcal{H}, d}$:

\begin{equation*}\label{eq:product_to_sum}
    \hat{L}_d = \frac{1}{2} \sum_{i=0}^K \sum_{j=0}^K \theta_i^{\mathcal{L}, d} \, \theta_j^{\mathcal{H}, d} \left[T_{i+j}\left(\tilde{L}_d\right) + T_{|i-j|}\left(\tilde{L}_d\right)\right].
\end{equation*}
\end{proposition}

Prop. \ref{prop:2} allows to reduce the computational and memory overhead. Rather than performing direct matrix multiplication, the matrix terms $T_k(\tilde{L}_d)$ are summed and multiplied by the scalar coefficients $\theta_i^{\mathcal{L}, d}$ and $\theta_i^{\mathcal{H},d}$.

\subsection{Stability of the Composed Chebyshev Filter}
\label{sec:stability_product_filter}

We now provide a stability analysis of the proposed composed operator $\hat{L}_d = f_d^{\mathcal{L}}(\tilde{L}_d)\, f_d^{\mathcal{H}}(\tilde{L}_d)$, where $\tilde{L}_d = 2\,L_d / \lambda_{d}^{\text{max}} - I$ is the rescaled Laplacian used in Sec.~\ref{sec:product_filter}. Throughout, $\|\cdot\|_2$ denotes the spectral norm (largest singular value) and $\|\cdot\|_F$ denotes the Frobenius norm.

\paragraph{Why rescaling matters} When $L_d$ is symmetric (e.g., undirected graphs with symmetric normalization), it admits an eigendecomposition $L_d = U_d \Lambda_d U_d^\top$ with real eigenvalues $\lambda_{d}^{(i)}\in[0,\lambda_{d}^{\text{max}}]$. Therefore, $\tilde{L}_d = 2\,L_d/\lambda_{d}^{\text{max}} - I$ has eigenvalues $\tilde{\lambda}_{d}^{(i)} = 2\,\lambda_{d}^{(i)}/\lambda_{d}^{\text{max}} - 1 \in [-1,1]$. This is the classical regime where Chebyshev polynomials satisfy $|T_k(x)|\le 1$ for $x\in[-1,1]$, preventing the exponential growth that occurs when $|x|>1$.

\begin{definition}[BIBO stability]
\label{def:bibo}
Fix a dimension $d\in\{1,\dots,D\}$. We say that a linear graph operator $\mathcal{T}_d\in\mathbb{R}^{N\times N}$ is \emph{bounded-input bounded-output (BIBO) stable} with respect to the Frobenius norm if there exists a constant $C_d<\infty$ such that, for every matrix-valued graph signal $S\in\mathbb{R}^{N\times C}$, we have:
\[
\|\mathcal{T}_d\,S\|_F \le C_d\,\|S\|_F.
\]
In \methodname, the relevant operators are Chebyshev polynomial filters $f_d(\tilde{L}_d)$ and the composed operator $\hat{L}_d=f_d^{\mathcal{L}}(\tilde{L}_d)\,f_d^{\mathcal{H}}(\tilde{L}_d)$ acting on the initial predictions $\hat{Y}_0\in\mathbb{R}^{N\times C}$.
\end{definition}

\begin{proposition}[Bounded Chebyshev bases. Proof in Appendix~\ref{proof:prop_cheb_basis_bound}]
\label{prop:cheb_basis_bound}
Let $\tilde{L}_d\in\mathbb{R}^{N\times N}$ be symmetric and admit the eigendecomposition $\tilde{L}_d = U_d \tilde{\Lambda}_d U_d^\top$, where $\tilde{\Lambda}_d=\mathrm{diag}(\tilde{\lambda}_{d}^{(1)},\dots,\tilde{\lambda}_{d}^{(N)})$ satisfies $\tilde{\lambda}_{d}^{(i)}\in[-1,1]$ for all $i\in\{1,\dots,N\}$.
Then for all integers $k\ge 0$, we have:
\[
\|T_k(\tilde{L}_d)\|_2 \le 1.
\]
\end{proposition}

Prop.~\ref{prop:cheb_basis_bound} implies that finite-order Chebyshev filters are bounded by the $\ell_1$ magnitude of their coefficients, yielding a BIBO-type stability guarantee.

\begin{proposition}[BIBO stability of Chebyshev polynomial filters. Proof in Appendix~\ref{proof:prop_bibo_cheb}]
\label{prop:bibo_cheb}
Let $f_d(\tilde{L}_d)=\sum_{k=0}^{K}\alpha_k \,T_k(\tilde{L}_d)$, where $\tilde{L}_d$ is symmetric and its eigenvalues satisfy $\tilde{\lambda}_{d}^{(i)}\in[-1,1]$ for all $i$. Then
\[
\|f_d(\tilde{L}_d)\|_2 \le \sum_{k=0}^{K}|\alpha_k|
\quad\text{and}\quad
\|f_d(\tilde{L}_d)\,S\|_F \le \Big(\sum_{k=0}^{K}|\alpha_k|\Big)\,\|S\|_F,
\quad \forall S\in\mathbb{R}^{N\times C}.
\]
\end{proposition}

\begin{corollary}[Stability of the product-composed filter. Proof in Appendix~\ref{proof:coro_product_stability}]
\label{coro:product_stability}
For each dimension $d$, let $f_d^{\mathcal{L}}(\tilde{L}_d)=\sum_{k=0}^{K}\theta_k^{\mathcal{L},d}\,T_k(\tilde{L}_d)$ and $f_d^{\mathcal{H}}(\tilde{L}_d)=\sum_{k=0}^{K}\theta_k^{\mathcal{H},d}\,T_k(\tilde{L}_d)$. Define $\hat{L}_d=f_d^{\mathcal{L}}(\tilde{L}_d)\, f_d^{\mathcal{H}}(\tilde{L}_d)$. Then
\begin{align*}
\|\hat{L}_d\|_2
&\le
\Big(\sum_{k=0}^{K}|\theta_k^{\mathcal{L},d}|\Big)\,
\Big(\sum_{k=0}^{K}|\theta_k^{\mathcal{H},d}|\Big), \\
\|\hat{L}_d\,S\|_F
&\le
\Big(\sum_{k=0}^{K}|\theta_k^{\mathcal{L},d}|\Big)\,
\Big(\sum_{k=0}^{K}|\theta_k^{\mathcal{H},d}|\Big)\,
\|S\|_F,
\qquad \forall S\in\mathbb{R}^{N\times C}.
\end{align*}
Moreover, by Prop.~\ref{prop:2}, the product $\hat{L}_d$ can be written as a \emph{single} Chebyshev polynomial of degree $2K$, i.e.,
\[
\hat{L}_d
=
\sum_{r=0}^{2K} \bar{\theta}_{r}^{\,d}\, T_r(\tilde{L}_d),
\]
where $\bar{\theta}^{\,d}=[\bar{\theta}_{0}^{\,d},\dots,\bar{\theta}_{2K}^{\,d}]^\top\in\mathbb{R}^{2K+1}$ denotes the resulting coefficient vector. This induced vector satisfies
\[
\|\bar{\theta}^{\,d}\|_1 \le \|\theta^{\mathcal{L},d}\|_1 \, \|\theta^{\mathcal{H},d}\|_1,
\]
with $\theta^{\mathcal{L},d}=[\theta_0^{\mathcal{L},d},\dots,\theta_K^{\mathcal{L},d}]^\top$ and similarly for $\theta^{\mathcal{H},d}$.
\end{corollary}

\paragraph{Stability of the full update rule} Recall that \methodname\ predicts, for each dimension $d$, the score matrix $S_d = \hat{L}_d\,\hat{Y}_0\,H_d\in\mathbb{R}^{N\times C}$ and then applies a \emph{row-wise} softmax (Eq.~(\ref{eq:gnn_update_rule})) to obtain $\hat{Y}_d=\mathrm{softmax}(S_d)$. Using the submultiplicativity of induced norms, we obtain the bound:
\begin{equation}
\|S_d\|_F \le \|\hat{L}_d\|_2 \, \|\hat{Y}_0\|_F \, \|H_d\|_2.
\label{eq:stability_score_bound}
\end{equation}
Combining \eqref{eq:stability_score_bound} with Cor.~\ref{coro:product_stability} shows that the pre-softmax logits remain bounded whenever the coefficient $\ell_1$ norms and $\|H_d\|_2$ are controlled.

\begin{proposition}[Stability of the row-wise softmax. Proof in Appendix~\ref{proof:prop_softmax_stability}]
\label{prop:softmax_stability}
Let $S,S'\in\mathbb{R}^{N\times C}$ and define $Y=\mathrm{softmax}(S)$ and $Y'=\mathrm{softmax}(S')$,
where $\mathrm{softmax}(\cdot)$ is applied row-wise.
Then
\[
\|Y\|_F \le \sqrt{N}
\quad\text{and}\quad
\|Y-Y'\|_F \le \frac{1}{2}\,\|S-S'\|_F.
\]
\end{proposition}

As an immediate consequence, small perturbations of the score matrix translate into controlled changes in the predicted probabilities. In particular, for two initial predictions $\hat{Y}_0$ and $\hat{Y}'_0$, letting
$S_d=\hat{L}_d\hat{Y}_0H_d$ and $S_d'=\hat{L}_d\hat{Y}_0'H_d$, we obtain:
\[
\|\hat{Y}_d-\hat{Y}_d'\|_F
=
\|\mathrm{softmax}(S_d)-\mathrm{softmax}(S_d')\|_F
\le
\frac{1}{2}\,\|\hat{L}_d\|_2\,\|H_d\|_2\,\|\hat{Y}_0-\hat{Y}_0'\|_F.
\]

\paragraph{Why the product is robust to high-frequency noise} Applying Prop.~\ref{prop.A} to $L=\tilde{L}_d$, the composed spectral response satisfies $f_d(\tilde{\lambda}) = f_d^{\mathcal{L}}(\tilde{\lambda})\, f_d^{\mathcal{H}}(\tilde{\lambda})$. Hence, a frequency component is amplified only if \emph{both} branches assign it a large gain, yielding a soft gating effect. The following result formalizes bandwise attenuation.

\begin{proposition}[Bandwise noise attenuation of the product. Proof in Appendix~\ref{proof:prop_bandwise_attenuation}]
\label{prop:bandwise_attenuation}
Let $\tilde{L}_d=U_d \tilde{\Lambda}_d U_d^\top$ with $\tilde{\Lambda}_d=\mathrm{diag}(\tilde{\lambda}_{d}^{(1)},\dots,\tilde{\lambda}_{d}^{(N)})\subseteq[-1,1]$. For any index set $\Omega\subseteq\{1,\dots,N\}$, define the spectral projector $P_{\Omega,d} = U_d\,\mathrm{diag}(\mathbf{1}_{i\in\Omega})\,U_d^\top$. Let $f_d^{\mathcal{L}}(\tilde{L}_d)$ and $f_d^{\mathcal{H}}(\tilde{L}_d)$ be two spectral filters and define
$f_d(\tilde{L}_d)=f_d^{\mathcal{L}}(\tilde{L}_d)\,f_d^{\mathcal{H}}(\tilde{L}_d)$. Then for any $x\in\mathbb{R}^{N}$, we have:
\[
\|P_{\Omega,d}\, f_d(\tilde{L}_d)\, x\|_2
\le
\Big(\max_{i\in\Omega}\big|f_d^{\mathcal{L}}(\tilde{\lambda}_{d}^{(i)})\,f_d^{\mathcal{H}}(\tilde{\lambda}_{d}^{(i)})\big|\Big)\,
\|P_{\Omega,d}\,x\|_2.
\]
In particular, if $f_d^{\mathcal{L}}$ is low-pass so that $\max_{i\in\Omega_{\mathrm{high}}}|f_d^{\mathcal{L}}(\tilde{\lambda}_{d}^{(i)})|\le \varepsilon_{\mathrm{high}}$ on a high-frequency band $\Omega_{\mathrm{high}}$, then
\[
\max_{i\in\Omega_{\mathrm{high}}}\big|f_d^{\mathcal{L}}(\tilde{\lambda}_{d}^{(i)})\,f_d^{\mathcal{H}}(\tilde{\lambda}_{d}^{(i)})\big|
\le
\varepsilon_{\mathrm{high}} \cdot \max_{i\in\Omega_{\mathrm{high}}}|f_d^{\mathcal{H}}(\tilde{\lambda}_{d}^{(i)})|.
\]
\end{proposition}

Prop.~\ref{prop:bandwise_attenuation} shows that the product can not arbitrarily amplify high-frequency noise when the low-pass branch attenuates that band. Compared to additive fusion $f_{d,\mathrm{sum}}(\tilde{\lambda})=\delta \, f_d^{\mathcal{L}}(\tilde{\lambda})+(1-\delta)\,f_d^{\mathcal{H}}(\tilde{\lambda})$, the product $f_d(\tilde{\lambda})=f_d^{\mathcal{L}}(\tilde{\lambda})\,f_d^{\mathcal{H}}(\tilde{\lambda})$ is conservative. If either branch suppresses a frequency, the composed response also suppresses it. This provides a principled explanation for the empirical robustness of the product operation compared with the sum or weighted-sum variants reported in the ablation study (Sec.~\ref{sec:ablation_Study}).

\subsection{Compatibility Matrices}\label{sec:compa_matrix}

The dimension-specific compatibility matrix $H_d$ captures the likelihood of connections between nodes from different classes in dimension $d$ \cite{zhu2021graph}. We define $H_d$ as a learnable $C \times C$ matrix, which is empirically initialized using the ground-truth labels. Let $I_{c}^{\text{train}}$ denote the set of indices for training nodes belonging to class $c$. 
For every pair of classes $c_1$ and $c_2$, we initialize $\left(H_d\right)_{c_1 c_2}$ with:
\begin{equation}\label{eq:comp_matrix_init}
    \left(H_d\right)_{c_1 c_2} = \frac{\sum_{i \in I_{c_1}^{\text{train}}, \ j \in I_{c_2}^{\text{train}}}\left(A_d\right)_{ij}}{\sum_{i,j=0}^{N} \left(A_d\right)_{ij}}.
\end{equation}
The initial weights $\left(H_d\right)_{c_1 c_2}$ are computed as the proportion of edges in dimension $d$ linking nodes of class $c_1$ with nodes of class $c_2$. During the training process, $H_d$ is composed with $\hat{Y}_0$ and $\hat{L}_d$ as indicated in Eq. (\ref{eq:gnn_update_rule}), and undergoes refinement through backpropagation of the loss function, thereby optimizing the model to the specificities of the prediction task. Importantly, to avoid violating the semi-supervised learning paradigm, only training labels are used to estimate $H_d$.

The compatibility matrix $H_d$ serves as a mechanism to integrate and adjust the homophily or heterophily level into the update rule of graph neural networks, diverging from traditional methods that typically employ a normal distribution for weight initialization. The initialization strategy in Eq. (\ref{eq:comp_matrix_init}) is beneficial for accounting for the varying levels of homophily across different parts of the multiplex graph. Specifically, the overall homophily level within dimension $d$ (\textit{i.e.}, $h_d$) can be estimated by averaging the diagonal elements of $H_d$.

\subsection{Sparse Consensus Labels}\label{sec:loss_labels} 
The matrix $\hat{Y}_0=\mathrm{MLP}(X)\in\mathbb{R}^{N\times C}$ defines the initial label predictions. For each dimension $d$, we form the dimension-specific score matrix:
\begin{equation}
S_d \;=\; \hat{L}_d \, \hat{Y}_0 \, H_d \;\in\; \mathbb{R}^{N\times C}.
\label{eq:score_matrix_Sd}
\end{equation}
To obtain a valid per-node class-probability vector, we apply a row-wise softmax. The matrix $\hat{Y}_d \in \mathbb{R}^{N \times C}$ represents the dimension-specific label predictions defined in Eq.~(\ref{eq:gnn_update_rule}). For each node $v_i$, the prediction $(\hat{Y}_d)_{i:}$ is:
\begin{equation}
(\hat{Y}_d)_{i:} = \big(\mathrm{softmax}(S_d)_{i:}\big)_c \;=\;
\frac{\exp\big((S_d)_{ic}\big)}{\sum_{c'=1}^{C}\exp\big((S_d)_{ic'}\big)}.
\label{eq:softmax_score_matrix_Sd}
\end{equation}
Accordingly, the per-node cross-entropy loss on dimension $d$ is:
\begin{equation}
\ell\big((S_d)_{i:}, Y_{i:}\big)
\;=\;
-\sum_{c=1}^{C} Y_{ic}\,\log\Big(\hat{Y}_d\Big)_{ic},
\label{eq:ce_loss_logits_form}
\end{equation}
where $Y_{i:}\in\{0,1\}^{C}$ denotes the one-hot ground-truth label vector of node $v_i$.

We train \methodname in a semi-supervised setting, aligning the ground-truth and predicted labels across each dimension. Let $I^{\text{train}}$ denote the set of indices of labeled training nodes. The minimized loss function is a $l_2$-regularized categorical cross-entropy loss as described below:
\begin{equation}\label{eq:xent_loss}
    \mathcal{J} = - \sum_{i \in I^{\text{train}}} \, \sum_{d=1}^D \, \ell\big((S_d)_{i:}, Y_{i:}\big) + \alpha \sum_{w \in \mathbb{W}} \norm{w}_2^2,
\end{equation}
where $\alpha$ is a balancing hyperparameter, and $\mathbb{W}$ is the set of trainable parameters of the Multilayer Perceptron.
 
The cross-entropy loss function enables the modeling of dimension-specific class information in  $\hat{Y}_d$, potentially resulting in distinct labels across different dimensions. To find the sparse consensus predictions $\hat{Y}$, we solve the optimization problem described in Eq. (\ref{eq:consensus_loss}) at the end of the training process based on Eq. (\ref{eq:xent_loss}).
\begin{equation}\label{eq:consensus_loss}
    \underset{\hat{Y} \in \mathbb{R}^{N \times C}}{\text{argmin}} \sum_{d=1}^{D} \norm{\hat{Y} - \hat{Y}_d}_2^2 + \beta \norm{\hat{Y}}_1.
\end{equation}
The first term in Eq. (\ref{eq:consensus_loss}) is a sum of Frobenius norms to minimize the distance between the consensus prediction $\hat{Y}$ and the dimension-specific predictions $\hat{Y}_d$. The second term is an $l_1$-norm regularization to induce sparsity in $\hat{Y}$, promoting solutions where nodes are classified with high likelihoods. The coefficient $\beta$ is a hyperparameter balancing the two terms.

The loss function associated with Eq. (\ref{eq:consensus_loss}) is a convex function consisting of the sum of two convex functions: the Frobenius norm term and the $l_1$-norm regularization. Consequently, we adopt a proximal-gradient optimization method to manage the non-smooth regularization. At each iteration, the consensus predictions $\hat{Y}$ are updated as follows:
 
\begin{equation}
    \hat{Y}^{(i+1)} = \text{prox}_{\beta t} \left(\hat{Y}^{(i)} - 2t \sum_{d=1}^D \left(\hat{Y}^{(i)} - \hat{Y}_d\right) \right),
\end{equation}
where $\hat{Y}^{(i)}$ represents the predictions at iteration $i$, $t = \frac{1}{4D} $ is the step size, and the proximal operator is:
\begin{equation}
    \text{prox}_\lambda(v) = \text{sign}(v) \max\left(\left|v\right| - \lambda, 0\right).
\end{equation}

\begin{algorithm}
\caption{\textit{\methodname}}
\label{algo:algo}
\begin{algorithmic}[1]

\REQUIRE multiplex graph $G$, features matrix $X$, indices of training node for every $c$ class $I_{c}^{\text{train}}$, degree of filters $K$, number of iterations $T_1$ and $T_2$.
\ENSURE  Sparse consensus label predictions $\hat{Y}$.\newline

\FOR{$d \gets 1$ to $D$}
    \STATE $L_d \gets I - \Delta_d^{-\frac{1}{2}} \cdot (A_d + I) \cdot \Delta_d^{-\frac{1}{2}}$
    \STATE $\tilde{L}_d \gets 2 \, L_d / \lambda_{d}^{\max} - I$
    \FOR{$c_1, c_2 \in \{1, \dots, C\}$}
        \STATE $\left(H_d\right)_{c_1 c_2} \gets \frac{\sum_{i \in I_{c_1}^{\text{train}}, \ j \in I_{c_2}^{\text{train}}}\left(A_d\right)_{ij}}{\sum_{i,j=0}^{N} \left(A_d\right)_{ij}}$
    \ENDFOR 
\ENDFOR\newline

\FOR{$epoch \gets 1$ to $T_1$}
    \STATE $\hat{Y}_0 \gets \text{MLP}(X)$ 
    \FOR{$d \gets 1$ to $D$}
        \STATE $\gamma_i^{\mathcal{L}, d} \gets \gamma_0^d - \sum_{j=1}^i \gamma_j^d$,  $\, \, \, \forall i \in \left\{1, \dots, K \right\}$
        \STATE $\gamma_i^{\mathcal{H}, d} \gets \sum_{j=0}^i\gamma_j^d$,  $\, \, \, \forall i \in \left\{1, \dots, K \right\}$

        \FOR{$k \gets 1$ to $K$}
            \STATE $\theta_k^{\mathcal{L}, d} \gets \frac{2}{K + 1} \sum_{j=0}^K \gamma_j^{\mathcal{L}, d} \:\: T_k\left(cos\left(\frac{j + 1/2}{K+1} \pi\right)\right)$
            \STATE $\theta_k^{\mathcal{H}, d} \gets \frac{2}{K + 1} \sum_{j=0}^K \gamma_j^{\mathcal{H}, d} \:\: T_k\left(cos\left(\frac{j + 1/2}{K+1} \pi\right)\right)$
        \ENDFOR\newline
        
        \STATE \small $\hat{L}_d \gets \frac{1}{2} \displaystyle\sum_{i=0}^K \sum_{j=0}^K \theta_i^{\mathcal{L}, d} \, \theta_j^{\mathcal{H}, d} \left[T_{i+j}\left(\tilde{L}_d\right) + T_{|i-j|}\left(\tilde{L}_d\right)\right]$
        \STATE $\hat{Y}_d \gets \mathrm{softmax}\big(\hat{L}_d \cdot \hat{Y}_0 \cdot H_d\big)$
    \ENDFOR\newline

    \STATE Update the parameters $\gamma^d$, $H_d$, and MLP weights via gradient descent to minimize the $l_2$-regularized categorical cross-entropy loss $\mathcal{J}$.
\ENDFOR\newline

\STATE $\hat{Y} \gets \textbf{0}$
\FOR{$i \gets 1$ to $T_2$}
    \STATE $\hat{Y}^{(i+1)} \gets \text{prox}_{\beta t} \left(\hat{Y}^{(i)} - 2 \, t \sum_{d=1}^D \left(\hat{Y}^{(i)} - \hat{Y}_d\right) \right)$
\ENDFOR\newline
\STATE \textbf{Return} $\hat{Y}$.
\end{algorithmic}
\end{algorithm}

\subsection{Algorithm \& Complexity}

Algorithm \ref{algo:algo} summarizes the proposed approach \methodname. The time complexity of the proposed model is $\mathcal{O}\left(DC^2\mathcal{E} + T \left(FMC + D\left(K^2 \mathcal{E} + NC\mathcal{E} \right) \right) \right)$, where $T$ is the number of iterations, $N$ is the number of nodes, $D$ is the number of dimensions, $C$ is the number of classes, $K$ is the degree of the polynomial filters $f_d^{\mathcal{L}}(\tilde{L}_d)$ and $f_d^{\mathcal{H}}(\tilde{L}_d)$, $M$ is the size of embeddings, $F$ is the size of the features, and $\mathcal{E}$ is the maximum number of edges in all dimensions. The memory complexity is $\mathcal{O}\left( D \left( K^2\mathcal{E} + C^2 + NC \right) + \eta \right)$. Both computational complexities are linear with respect to $N$ and $\mathcal{E}$. 


\subsection{Generalization Analysis}
\label{sec:generalization}

We now provide a statistical generalization bound for \methodname\ that quantifies how the expected classification risk relates to the empirical training risk, and how this gap depends on the operator norms of the dimension-specific propagation matrices $\hat{L}_d$ and compatibility matrices $H_d$.

\paragraph{Risk definitions}
Let $I^{\mathrm{train}}\subseteq\{1,\dots,N\}$ be the index set of labeled training nodes and let $n:=|I^{\mathrm{train}}|$. We define the empirical training risk averaged across dimensions as:
\begin{equation}
\widehat{\mathcal{R}}
\;=\;
\frac{1}{D\,n}\sum_{d=1}^{D}\;\sum_{i\in I^{\mathrm{train}}}
\ell\big((S_d)_{i:}, Y_{i:}\big).
\label{eq:empirical_risk_def}
\end{equation}
For the population risk, we adopt the standard learning-theoretic abstraction where training examples $(x,y)$ are drawn i.i.d.\ from an unknown distribution $\mathcal{D}$ over node features and labels.\footnote{This abstraction is commonly used to quantify how empirical performance translates to expected performance. In the semi-supervised node classification protocol used in Sec.~5, the labeled training nodes $I^{\mathrm{train}}$ can be viewed as a random labeled sample from the underlying node population.} Let $\mathcal{R}$ denote the expected risk averaged across dimensions:
\begin{equation}
\mathcal{R}
\;=\;
\frac{1}{D}\sum_{d=1}^{D}\;
\mathbb{E}_{(x,y)\sim\mathcal{D}}
\Big[\ell\big(s_d(x), y\big)\Big],
\label{eq:population_risk_def}
\end{equation}
where $s_d(x)\in\mathbb{R}^{C}$ denotes the dimension-$d$ score vector produced by the \methodname\ pipeline for an input feature vector $x$, consistent with the matrix in Eq.~(\ref{eq:score_matrix_Sd}).

\paragraph{Vector-valued Rademacher complexity} Let $\mathcal{F}_0$ denote the class of functions implemented by the MLP mapping $x\mapsto \hat{y}_0(x)\in\mathbb{R}^{C}$ (i.e., row-wise outputs of $\hat{Y}_0$). For a labeled sample $S=\{(x_i,y_i)\}_{i=1}^{n}$, define the empirical Rademacher complexity of $\mathcal{F}_0$ by:
\begin{equation}
\mathfrak{R}_{n}(\mathcal{F}_0)
\;=\;
\frac{1}{n}\,
\mathbb{E}_{\sigma}
\Bigg[
\sup_{f\in\mathcal{F}_0}
\sum_{i=1}^{n}\sum_{c=1}^{C}\sigma_{ic}\, f_c(x_i)
\Bigg],
\label{eq:rad_def}
\end{equation}
where $\{\sigma_{ic}\}$ are i.i.d.\ Rademacher random variables taking values in $\{-1,+1\}$.

\begin{lemma}[Lipschitzness and boundedness of softmax cross-entropy.
Proof in Appendix~\ref{proof:lemma_ce_lipschitz}]
\label{lemma:ce_lipschitz}
Let $\ell(\cdot,y)$ be the softmax cross-entropy in Eq.~(\ref{eq:ce_loss_logits_form}). Then, for any fixed one-hot label vector $y\in\{0,1\}^{C}$, the map $z\mapsto \ell(z,y)$ is $\sqrt{2}$-Lipschitz with respect to the Euclidean norm $\|\cdot\|_2$. Moreover, if $\|z\|_{\infty}\le B_z$ then $\ell(z,y)\le \log(C)+2B_z$.
\end{lemma}

\paragraph{Complexity scaling under graph propagation} Define the class of dimension-$d$ score functions (logits) obtained by composing the MLP with the linear propagation operators $\hat{L}_d$ and $H_d$:
\begin{equation}
\label{eq:F_d_definition}
\mathcal{F}_d
=
\left\{
\begin{aligned}
&x\mapsto s_d(x)\in\mathbb{R}^{C} \;:\; \\
&\quad s_d(\cdot)\ \text{is induced by}\ S_d \text{in Eq.~(\ref{eq:score_matrix_Sd})},\ 
\hat{Y}_0\ \text{generated by some}\ f\in\mathcal{F}_0
\end{aligned}
\right\}.
\end{equation}
The next result shows that the Rademacher complexity of $\mathcal{F}_d$ is controlled by the operator norms of $\hat{L}_d$ and $H_d$.

\begin{proposition}[Rademacher complexity under propagation.
Proof in Appendix~\ref{proof:prop_rad_scaling}]
\label{prop:rad_scaling}
Fix a dimension $d\in\{1,\dots,D\}$. Assume $\hat{L}_d\in\mathbb{R}^{N\times N}$ and $H_d\in\mathbb{R}^{C\times C}$ are fixed matrices. Then, for any labeled sample of size $n$, the corresponding empirical Rademacher complexity satisfies
\begin{equation}
\mathfrak{R}_{n}(\mathcal{F}_d)
\;\le\;
\|\hat{L}_d\|_2\;\|H_d\|_2\;\mathfrak{R}_{n}(\mathcal{F}_0).
\label{eq:rad_scaling_main}
\end{equation}
\end{proposition}

\begin{theorem}[Generalization bound for \methodname.
Proof in Appendix~\ref{proof:thm_generalization}]
\label{thm:generalization}
Assume the labeled training nodes form an i.i.d.\ sample of size $n$ from $\mathcal{D}$. Assume further that, for each dimension $d$, the score vectors are uniformly bounded: $\|s_d(x)\|_{\infty}\le B_{d}$ for all $x$. Then for any $\delta\in(0,1)$, with probability at least $1-\delta$, the following holds simultaneously for the dimension-averaged risk of \methodname:
\begin{equation}
\mathcal{R}
\;\le\;
\widehat{\mathcal{R}}
\;+\;
\frac{2\sqrt{2}}{D}\Big(\sum_{d=1}^{D}\|\hat{L}_d\|_2\,\|H_d\|_2\Big)\,
\mathfrak{R}_{n}(\mathcal{F}_0)
\;+\;
3\,B_{\max}\sqrt{\frac{\log(2/\delta)}{2n}},
\label{eq:gen_bound_main}
\end{equation}
where $B_{\max}:=\max_{d\in\{1,\dots,D\}}\big(\log(C)+2B_d\big)$.
\end{theorem}

\paragraph{Interpretation for \methodname}
Theorem~\ref{thm:generalization} shows that the generalization gap is governed by: (i) the base prediction complexity $\mathfrak{R}_{n}(\mathcal{F}_0)$ of the MLP, (ii) the dimension-specific amplification factors $\|\hat{L}_d\|_2\|H_d\|_2$, and (iii) the sample size $n=|I^{\mathrm{train}}|$. Importantly, by Corollary~\ref{coro:product_stability} (Sec.~\ref{sec:stability_product_filter}), the composed Chebyshev propagation operator satisfies:
\[
\|\hat{L}_d\|_2 \;\le\;
\Big(\sum_{k=0}^{K}|\theta_k^{\mathcal{L},d}|\Big)
\Big(\sum_{k=0}^{K}|\theta_k^{\mathcal{H},d}|\Big),
\]
which provides an explicit control of the generalization term in Eq.~(\ref{eq:gen_bound_main}) through the learned Chebyshev coefficients.

Moreover, the bounded-logit assumption can be connected to Sec.~\ref{sec:stability_product_filter}. By Eq.~(\ref{eq:stability_score_bound}) and $\|z\|_\infty\le\|z\|_2$, one has
$\|s_d(x)\|_\infty \le \|S_d\|_F \le \|\hat{L}_d\|_2\,\|\hat{Y}_0\|_F\,\|H_d\|_2$ whenever the MLP outputs $\hat{Y}_0$ are bounded.

\section{Experiments}

We conduct an empirical evaluation to show the suitability of the proposed approach for modeling heterophily and homophily for node classification in multiplex graphs. We compare \methodname\footnote{The code of \methodname is available at \href{https://drive.google.com/drive/folders/1ROhUghYARMRGzyLyBP2wXJH6MxXfu4yf?usp=drive_link}{this link}.} against state-of-the-art multiplex graph models, namely: GATNE \cite{cen2019representation}, mGCN \cite{ma2019multi}, SSDCM \cite{mitra2021semi}, DMGI \cite{park2020dmgi}, HDMI \cite{jing2021hdmi}, MGDCR \cite{mgdcr}, DMG \cite{mo2023disentangled}, X-GOAL \cite{jing2022x}, HMGE \cite{hmge2023}, and InfoMGF \cite{shen2024beyond}. In addition, we include recent unidimensional graph representation learning baselines that can handle heterophily. More precisely, our comparison includes PolyGCL \cite{chen2024polygcl} and TFE-GNN \cite{duan2024unifying} on a unified adjacency matrix $\tilde{A}$ using the original feature matrix $X$. We adapt these methods to the multiplex setting following the common practice of aggregating all multiplex dimensions into one unified adjacency matrix \(\tilde{A}\):
\begin{equation}
\tilde{A}=\frac{1}{D}\sum_{d=1}^{D}(A_d + I).
\end{equation}

\begin{table}
\caption{Data description.}
\centering
\resizebox{\linewidth}{!}{
\begin{tabular}{ c|c|c|c|c } 
\hline
Dataset & Synthetic & arXiv & Movies & Amazon  \\ 
\hline
\# Dims & 3 & 2 & 3 & 3 \\
\hline
\# Nodes & 9,600 & 169,343 & 10,589 & 7,621 \\
\hline
\# Edges & 343,494 & 7,879,585 &  485,962 & 1,386,799 \\
\hline
\# Attributes & 100 & 128 & 200 & 2,000 \\
\hline
\# Classes & 6 & 5 & 4 & 4 \\
\hline
$h_d$ & $\{0.1, 0.2, \dots, 0.9\}$ & 0.22 - 0.29 & 0.34 - 0.32 - 0.37 & 0.27 - 0.26 - 0.25 \\
\hline
\end{tabular}
}
\label{table:data_statistics}
\end{table}

\subsection{Datasets}

In this section, we describe the datasets used in the experiments. Table \ref{table:data_statistics} summarizes their characteristics, including the homophily ratio $h_d$ for each dimension.

\textit{Synthetic datasets:} We generate node classification labels and adjacency matrices using a method inspired by \cite{zhu2021graph}. The generated dimensions have a homophily ratio $h_d \in [0.1, \dots, 0.9]$. We describe the synthetic generation process in Sec. \ref{sec:synthetic_data_gen}. 

\textit{arXiv:} This dataset, inspired by \cite{lim2021large}, is based on the OGBN-arXiv network but contains different labels and 
two dimensions instead of one. The nodes represent papers, with edges connecting papers with citation and co-authorship relations. The node features are derived from the word2vec features of titles and abstracts. The class labels are to the publication year of the paper.

\textit{Movies:} This dataset is extracted from the film-director-actor-writer network in \cite{tang2009social}. The nodes represent movies, and edges connect movies that share directors, actors, or writers. Node features are token count vectorizations of the movie descriptions. The class labels correspond to the year the movie was released.

\textit{Amazon:} This dataset \cite{he2016ups} consists of Amazon items, with features being bag-of-words of item descriptions. There are three types of relations between items: \textit{also viewed}, \textit{also bought}, and \textit{bought together}. The classes are the items' categories (\textit{e.g.}, beauty and baby products). 

\subsection{Evaluation Protocol \& Parameter Settings} \label{protocol}

We evaluate \methodname\ and all baseline methods on the task of node classification over both synthetic and real-world multiplex graph datasets. For unsupervised representation learning methods, including HDMI, HMGE, GATNE, DMG, X-GOAL, InfoMGF, and PolyGCL, we first learn node embeddings without using labels and then train a logistic regression classifier on top of these embeddings to predict node labels. For supervised or semi-supervised methods, including DMGI, SSDCM, MGDCR, mGCN, TFE-GNN, and \methodname, class predictions are obtained directly from the model outputs. All experiments are performed using the same training/validation/test splits.

We report F1-Macro and F1-Micro scores by comparing predicted labels against the ground-truth labels. Each experiment is repeated five times. For \methodname, we report the mean and standard deviation across runs, whereas for the baseline methods we report the best result among five runs. For \methodname, the embedding dimension is fixed to $64$. The degree of the Chebyshev polynomial filters $K$ is set to $\{5, 4, 2, 3\}$ for the synthetic datasets, arXiv, Movies, and Amazon, respectively. We optimize the model using the Adam optimizer with a learning rate of $0.001$ and an $\ell_2$ weight decay coefficient $\alpha = 10^{-5}$. Training is performed for a maximum of $1{,}000$ epochs with early stopping if the validation performance does not improve for $100$ consecutive epochs. The $\ell_1$ regularization parameter $\beta$ is set to $1.0$ and we optimize the loss in Eq.~(\ref{eq:consensus_loss}).

\subsection{Experiments on Synthetic Datasets}

We first evaluate \methodname against baseline methods on node classification tasks using synthetic datasets. The experiments are divided into two parts: \textbf{(i)} constant homophily ratios and \textbf{(ii)} variable homophily ratios across the dimensions of the multiplex graph. Before that, we describe the synthetic data generation process.

\subsubsection{Synthetic Data Generation}\label{sec:synthetic_data_gen}

We generate node classification labels and adjacency matrices using a method inspired by \cite{zhu2021graph}, which extends the Barabási-Albert model with configurable class settings. First, nodes are randomly assigned into $C$ classes, keeping a balanced distribution. After that, node features are attributed using the features of the \textit{ogbn-products} dataset from Open Graph Benchmark (OGBN) \cite{hu2020open}, which is a product co-purchasing graph. For the edges, we initialize a compatibility matrix $\mathcal{B}_d$ that controls the homophily and heterophily settings of each dimension $d$, resulting in an overall homophily ratio $\rho_d$. The diagonal elements of $\mathcal{B}_d$ are set to the same value $\rho_d$, while the off-diagonal elements are set following the approach in \cite{abu2019mixhop}. The matrices $\mathcal{B}_d$ are employed to sample edges. Let $v_i$ and $v_j$ be nodes of class $c_{v_i}$ and $c_{v_j}$. The edge $(v_i, v_j)$ is added to dimension $d$ with probability $\left(\mathcal{B}_d\right)_{c_{v_i} c_{v_j}}$. This process results in a multiplex graph with $D$ dimensions, such that each dimension has a homophily ratio equal to $\rho_d$. In our experiments, we generate \textbf{(1)} multiplex graphs where $\rho_d$ is the same for all dimensions and \textbf{(2)} multiplex graphs where $\rho_d$ varies from one dimension to another.

\begin{figure}
\centering
\begin{subfigure}{0.6\linewidth}
    \centering
    \includegraphics[width=\textwidth]{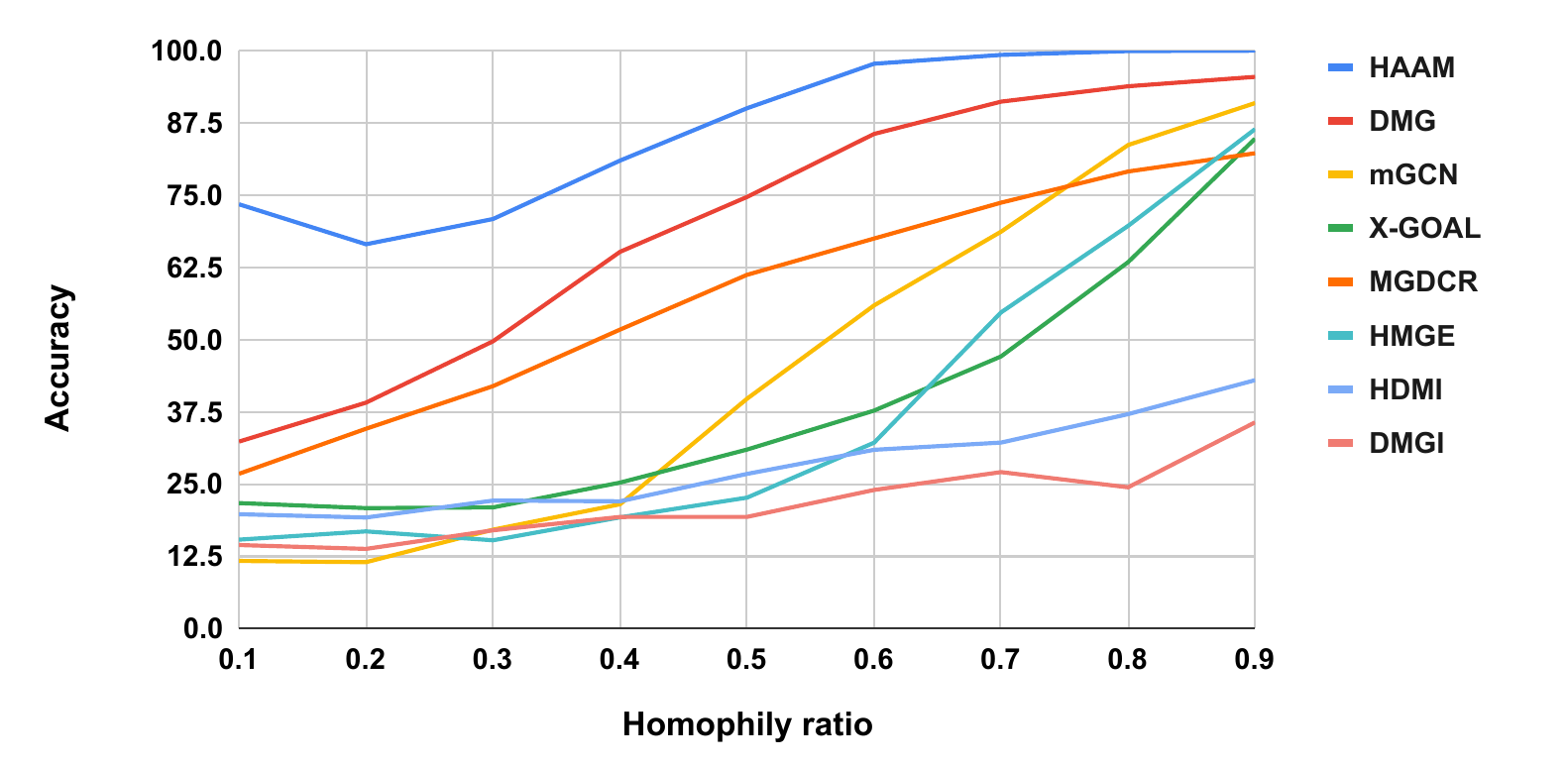}
    \caption{Constant homophily ratios across all dimensions.}\label{fig:synthetic_constant_h}
\end{subfigure}
\begin{subfigure}{0.6\linewidth}
    \centering
    \includegraphics[width=\textwidth]{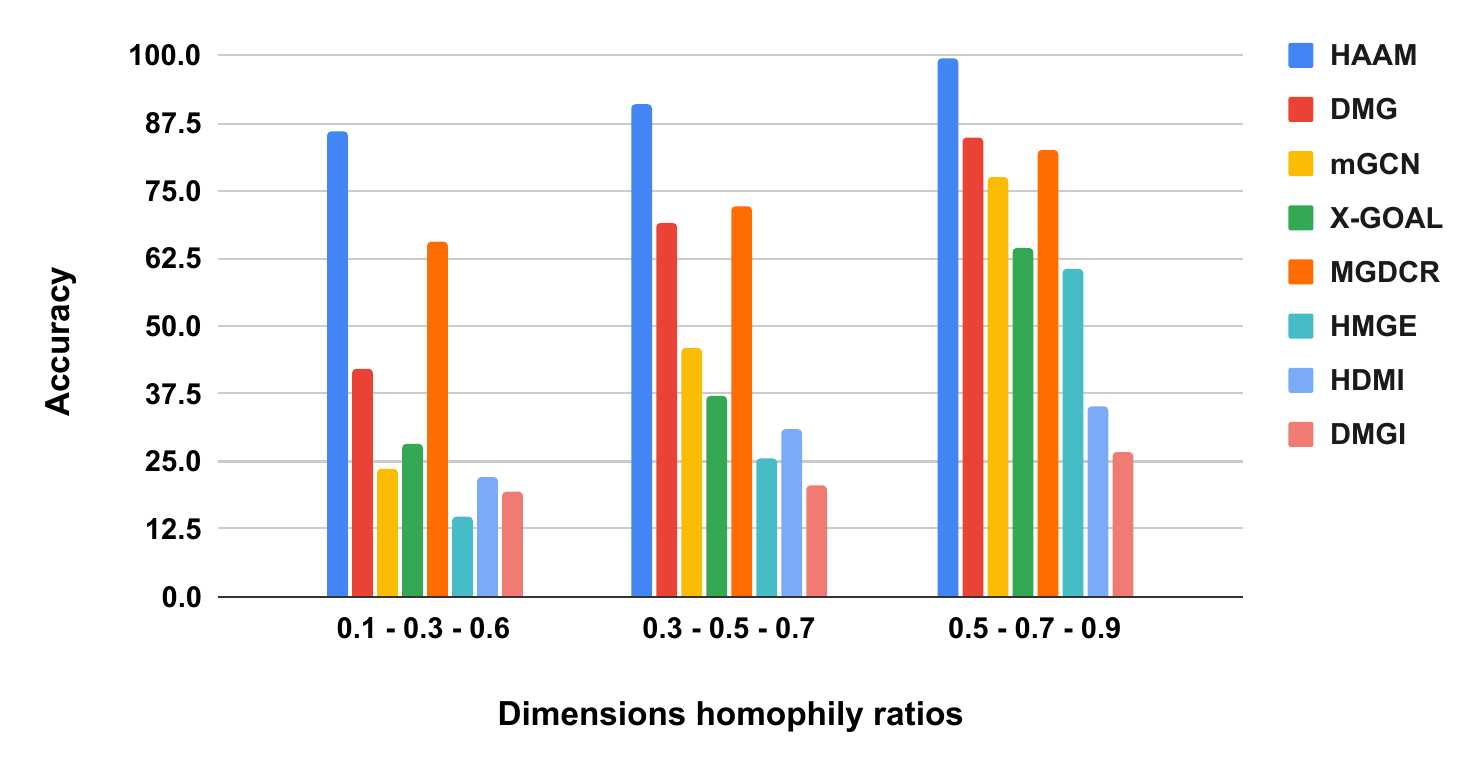}
    \caption{Variable homophily ratios in the same graph.}\label{fig:synthetic_variable_h}
\end{subfigure}

\caption{Results of node classification on synthetic datasets.}
\label{fig:synthetic_results}
\end{figure}

\subsubsection{Constant Homophily Ratios}

Figure~\ref{fig:synthetic_constant_h} shows a comparison of multiplex graph methods on datasets with increasing homophily ratios. For each value of $h \in {0.1, 0.2, \dots, 0.9}$, we generate a synthetic multiplex graph where $h_d = h$ in all dimensions $d$, following the methodology described in Sec.~\ref{sec:synthetic_data_gen}. We then train each method and measure accuracy on the resulting synthetic graphs. The results indicate that \methodname generally achieves higher accuracy than the compared methods across the entire range of $h$. As $h$ increases, the accuracy of \methodname also improves, approaching very high values when $h$ is close to $0.9$. Other methods such as HMGE, X-GOAL, and mGCN obtain lower performance, particularly at smaller homophily ratios. DMG and MGDCR show moderate improvements as homophily increases but tend to underperform compared to \methodname, and they face more difficulties when $h$ is low, which points to their limitations in settings with stronger heterophily. Overall, the results suggest that while some baselines can adapt to higher homophily, \methodname maintains strong performance across the full spectrum of $h$, making it a competitive approach in both low and high homophily regimes.

\subsubsection{Variable Homophily Ratios}

Fig. \ref{fig:synthetic_variable_h} groups results by the value of $h_d$ for each of the three dimensions of the graphs ($0.1-0.3-0.6$, $0.3-0.5-0.7$, and $0.5-0.7-0.9$). We follow the same generation protocol described in Sec. \ref{sec:synthetic_data_gen}, but assign different values of $h_d$ depending on the dimension. The objective is to evaluate competing algorithms when $h_d$ varies from one dimension to another. In this setting, \methodname remains competitive and consistent across all ranges, particularly in the lower homophily range $(0.1-0.3-0.6)$, where the performance gap becomes more evident. The combination of adaptive filter products and learnable dimension-specific compatibility matrices provides a flexible modeling approach, allowing \methodname to better capture varying levels of heterophily within the same multiplex graph. Moreover, consensus labels support the integration of information from different dimensions. Overall, \methodname demonstrates robustness across a wide range of homophily levels. In contrast, methods such as DMGI and HMGE tend to show weaker performance in low homophily settings, reflecting their limitations in modeling heterophily.

\subsection{Experiments on Real-World Datasets}

\begin{table}
\caption{Results of node classification on real-world datasets.}
\centering
\resizebox{\linewidth}{!}{
\begin{tabular}{ l|c|cc|cc|cc }
\hline
Method & Dimension & \multicolumn{2}{c|}{arXiv} & \multicolumn{2}{c|}{Movies} & \multicolumn{2}{c}{Amazon} \\
\cline{3-8}
       &      & F1-Macro & F1-Micro & F1-Macro & F1-Micro & F1-Macro & F1-Micro \\
\hline

\multicolumn{8}{l}{\textit{Unidimensional graph methods}} \\
PolyGCL & Single & 23.00 & 35.32 & 35.53 & 37.81 & 68.37 & 69.53 \\
TFE-GNN & Single & 18.71 & 21.95 & 39.44 & 41.22 & 20.10 & 29.63 \\
\hline

\multicolumn{8}{l}{\textit{Multiplex graph methods}} \\
DMGI   & Multiplex & 15.01 & 28.92 & 30.19 & 31.37 & 74.58 & 74.81 \\
HDMI   & Multiplex & OOM   & OOM   & \underline{40.25} & \underline{41.42} & 84.13 & 84.29 \\
SSDCM  & Multiplex & OOM   & OOM   & 30.14 & 30.53 & 82.20 & 82.40 \\
HMGE   & Multiplex & 31.64 & 44.03 & 33.86 & 37.56 & 24.84 & 37.80 \\
GATNE  & Multiplex & 23.56 & 34.28 & 29.62 & 31.40 & 19.11 & 30.90 \\
MGDCR  & Multiplex & 31.91 & 43.42 & 39.49 & 40.10 & 70.10 & 70.19 \\
DMG    & Multiplex & \underline{35.45} & 44.08 & 39.49 & 40.42 & 55.72 & 57.22 \\
X-GOAL & Multiplex & 28.72 & 40.17 & 39.55 & 40.57 & \underline{85.70} & \underline{85.79} \\
mGCN   & Multiplex & 33.03 & \underline{44.45} & 38.63 & 40.61 & 74.86 & 75.13 \\
InfoMGF& Multiplex & 23.11 & 36.49 & 39.28 & 40.54 & 66.12 & 65.87 \\
\hline

\methodname\ (mean) & Multiplex & \textbf{39.95} & \textbf{48.44} & \textbf{41.81} & \textbf{42.39} & \textbf{88.32} & \textbf{88.37} \\
\methodname\ (std)  & Multiplex & 0.26 & 0.09 & 0.37 & 0.37 & 0.25 & 0.31 \\
\hline
\end{tabular}
}
\label{table:node_classification}
\end{table}

Table~\ref{table:node_classification} reports the results of node classification on real-world datasets. The table is organized into two blocks to provide a comprehensive comparison: \textbf{(i)} \textit{unidimensional graph methods} (PolyGCL and TFE-GNN), which are originally designed for single graphs and are therefore evaluated on the aggregated adjacency matrix $\tilde{A}$, and \textbf{(ii)} \textit{native multiplex graph methods}, which directly operate on all dimensions $\{A_d\}_{d=1}^{D}$.

Across all datasets, \methodname\ achieves the best F1-Macro and F1-Micro scores, with consistently low standard deviations, indicating stable performance. On arXiv, \methodname\ reaches $\text{F1-Macro}=39.95$ and $\text{F1-Micro}=48.44$, outperforming the strongest multiplex competitors DMG (second-best Macro: $35.45$) and mGCN (second-best Micro: $44.45$) by $+4.50$ and $+3.99$, respectively. We also observe that HDMI and SSDCM run out of memory on this large dataset, whereas \methodname\ remains feasible.

On Movies, \methodname\ obtains $41.81/42.39$ (Macro/Micro) and improves over the best baseline HDMI ($40.25/41.42$) by $+1.56$ (Macro) and $+0.97$ (Micro). On Amazon, which exhibits low homophily across all dimensions, \methodname\ achieves the largest improvements: $88.32/88.37$ versus the strongest baseline X-GOAL ($85.70/85.79$), i.e., $+2.62$ (Macro) and $+2.58$ (Micro). These trends illustrate the limitations of existing multiplex methods that do not explicitly account for heterophily patterns that vary by relation. Thus, the obtained results empirically support our main design choices: (i) learning dimension-specific compatibility matrices to model class couplings that differ across dimensions, (ii) composing low-pass and high-pass Chebyshev filters via the proposed product mechanism to jointly exploit homophilic and heterophilic signals, and (iii) producing a sparse consensus prediction using proximal-gradient optimization.

Although PolyGCL and TFE-GNN are designed to handle heterophily in unidimensional graphs, their adaptation to multiplex graphs via $\tilde{A}$ generally underperforms native multiplex models, especially on arXiv and Amazon. This suggests that aggregating dimensions can discard relation-specific homophily/heterophily structure that is critical for prediction. In contrast, InfoMGF is more competitive on Movies but remains substantially below \methodname\ on arXiv and Amazon. These results indicate that multiplex structure learning alone is insufficient when heterophily must be explicitly modeled during propagation and prediction. Overall, Table~\ref{table:node_classification} provides evidence that \methodname\ captures the interplay of homophilic and heterophilic interactions across multiplex dimensions and yields improved node classification performance compared to both recent unidimensional heterophily methods and state-of-the-art multiplex baselines.

\begin{table}
\caption{Ablation study on node classification.}
\centering
\resizebox{\linewidth}{!}{
\begin{tabular}{ c|c|c|c|c|c|c } 
\hline
Dataset & \multicolumn{2}{|c|}{arXiv} & \multicolumn{2}{|c|}{Movies} & \multicolumn{2}{|c}{Amazon} \\

\hline
Metrics & F1-Macro & F1-Micro & F1-Macro & F1-Micro & F1-Macro & F1-Micro \\
\hline
Naive model & 33.03 & 44.45 & 38.63 & 40.61 & 74.86 & 75.13 \\
HAAM-CM ($H$)  & 36.95 & 47.32 & 39.38 & 41.05 & 86.62 & 86.79 \\
HAAM-CM ($H_d$)  & 38.97 & 47.05 & 40.12 & 41.32 & 87.17 & 87.42 \\
HAAM-CM ($H_d$ + \text{prox})  & 39.51 & 47.66 & 40.53 & 41.55 & 87.28 & 87.51 \\
HAAM-LP ($H_d + f_d^{\mathcal{L}} + \text{prox}$) & 38.37 & 47.32 & 41.19 & \underline{42.21} &  86.04 & 86.45 \\
HAAM-HP ($H_d + f_d^{\mathcal{H}} + \text{prox}$) & \textbf{40.63} & \underline{48.38} &  40.44 & 41.53 & \underline{88.26} & \underline{88.28} \\
HAAM-SUM ($H_d + f_d^{\mathcal{L}} + f_d^{\mathcal{H}} + \text{prox}$) & 38.27 & 47.66 & 40.49 & 41.53 & 87.85 & 87.91 \\
HAAM-SUM ($H_d + \delta \,f_d^{\mathcal{L}} + (1-\delta) \,f_d^{\mathcal{H}} + \text{prox}$) & 39.21 & 48.06 & \underline{41.49} & 42.03 & 88.18 & 88.28 \\
\hline
HAAM-Prod ($H_d + f_d^{\mathcal{H}} \cdot f_d^{\mathcal{L}} + \text{prox}$) & \underline{39.95} & \textbf{48.44} & \textbf{41.81} & \textbf{42.39} & \textbf{88.32} & \textbf{88.37} \\
\hline
\end{tabular}
}
\label{table:ablation_study}
\end{table}

\subsection{Ablation Study}
\label{sec:ablation_Study}

The ablation study in Table~\ref{table:ablation_study} provides a fine-grained analysis of the main components of \methodname\ by isolating: \textbf{(i)} compatibility modeling (shared vs.\ dimension-specific), \textbf{(ii)} sparse consensus via proximal optimization, and \textbf{(iii)} the spectral filtering design (low-pass, high-pass, their sum, weighted sum, and product). All variants are trained and evaluated under the same protocol and data splits described in Sec.~\ref{protocol}. 

\paragraph{Ablation variants}
We consider the following models:
\begin{itemize}
    \item \textbf{Naive model:} a per-dimension two-layer GCN baseline that does not use compatibility matrices, spectral filter compositions, or the proximal consensus mechanism.
    \item \textbf{HAAM-CM} ($H$): \methodname\ with a single shared compatibility matrix $H \in \mathbb{R}^{C \times C}$ across all dimensions, testing whether a global class-compatibility structure is sufficient.
    \item \textbf{HAAM-CM} ($H_d$): \methodname\ with dimension-specific compatibility matrices $\{H_d\}_{d=1}^{D}$, capturing relation-dependent class interactions.
    \item \textbf{HAAM-CM} ($H_d + \text{prox}$): HAAM-CM ($H_d$) augmented with the proximal-gradient consensus optimization (Eq.~(\ref{eq:consensus_loss})), yielding sparse consensus predictions.
    \item \textbf{HAAM-LP} ($H_d + f_d^{\mathcal{L}} + \text{prox}$): HAAM-CM ($H_d + \text{prox}$) with only the low-pass Chebyshev filter $f_d^{\mathcal{L}}$.
    \item \textbf{HAAM-HP} ($H_d + f_d^{\mathcal{H}} + \text{prox}$): HAAM-CM ($H_d + \text{prox}$) with only the high-pass Chebyshev filter $f_d^{\mathcal{H}}$.
    \item \textbf{HAAM-SUM} ($H_d + f_d^{\mathcal{L}} + f_d^{\mathcal{H}} + \text{prox}$): HAAM-CM ($H_d + \text{prox}$) where the two filter responses are combined by an unweighted sum, i.e., $f_d^{\mathcal{L}} + f_d^{\mathcal{H}}$.
    \item \textbf{HAAM-SUM} ($H_d + \delta f_d^{\mathcal{L}} + (1-\delta) f_d^{\mathcal{H}} + \text{prox}$): HAAM-CM ($H_d + \text{prox}$) with a weighted sum of low-/high-pass responses, where $\delta \in [0,1]$ is selected on the validation set.
    \item \textbf{HAAM-Prod} ($H_d + f_d^{\mathcal{H}} \cdot f_d^{\mathcal{L}} + \text{prox}$): the full model, where the low-pass and high-pass filters are combined by the proposed product (composition) mechanism.
\end{itemize}

\paragraph{Impact of compatibility modeling} As we can see from Table~\ref{table:ablation_study}, comparing the naive model with compatibility-based variants highlights that compatibility modeling is a primary driver of the improvements, especially on low-homophily datasets such as Amazon (Table~\ref{table:data_statistics}). For example, moving from the Naive model to HAAM-CM improves performance on Amazon (F1-Macro: $74.86 \rightarrow 86.62$ with shared $H$, and $74.86 \rightarrow 87.17$ with dimension-specific $H_d$). This indicates that explicitly modeling cross-class couplings is critical when edges frequently connect dissimilar labels.

\paragraph{Shared vs.\ dimension-specific compatibility matrices} Replacing a single shared compatibility matrix $H$ with dimension-specific matrices $H_d$ yields consistent gains on Movies and Amazon (e.g., Amazon F1-Micro: $86.79 \rightarrow 87.42$). These results support the hypothesis that different relations encode different class-interaction patterns. On arXiv, the two designs are competitive (with a small trade-off between Macro and Micro), which suggests that some datasets may benefit from partial sharing. However, the dimension-specific design provides a stronger and more flexible inductive bias overall. This observation is also consistent with the qualitative differences across dimensions shown in Fig.~\ref{fig:compatibility_matrices}.

\paragraph{Effect of proximal consensus optimization} Adding proximal consensus (HAAM-CM ($H_d + \text{prox}$)) improves over HAAM-CM ($H_d$) on all three datasets. These results indicate that explicitly reconciling the dimension-wise predictions into a sparse consensus label distribution is beneficial. Concretely, on arXiv we observe improvements (F1-Macro: $38.97 \rightarrow 39.51$, F1-Micro: $47.05 \rightarrow 47.66$), and similar gains appear on Movies and Amazon. These results empirically validate the role of the sparsity-inducing consensus mechanism in stabilizing predictions across dimensions.

\paragraph{Low-pass vs.\ high-pass filtering} The relative performance of HAAM-LP and HAAM-HP depends on the dataset homophily regime (Table~\ref{table:data_statistics}). On Movies, which exhibits comparatively higher homophily than Amazon, HAAM-LP is stronger than HAAM-HP (e.g., Movies F1-Micro: $42.21$ vs.\ $41.53$), indicating that smoothing low-frequency components is particularly beneficial. Conversely, on Amazon and arXiv, both characterized by lower homophily, HAAM-HP yields larger gains (e.g., Amazon F1-Macro: $88.26$), highlighting the importance of preserving high-frequency (heterophilic) signals. This behavior is also consistent with the spectral response visualizations reported in Figs.~\ref{fig:filter_response_amazon} and \ref{fig:filter_response_movies}.

\paragraph{Sum vs.\ weighted sum vs.\ product} Among the fusion strategies, HAAM-SUM provides improvements over compatibility-only variants but is generally less robust than either (i) tuning a weighted mixture or (ii) using the proposed product. Introducing the validation-selected $\delta$ in the weighted sum improves over the unweighted sum (e.g., Amazon F1-Macro: $87.85 \rightarrow 88.18$), showing that balancing low-/high-frequency contributions matters. However, HAAM-Prod achieves the strongest and most consistent performance across datasets, obtaining the best results on five out of six metrics and remaining competitive on the remaining one (arXiv F1-Macro, where HAAM-HP is highest). Overall, these results support our design choice. More precisely, the product-based composition provides a robust mechanism to jointly exploit homophilic (low-frequency) and heterophilic (high-frequency) information without requiring manual tuning of a mixture coefficient, and it tends to generalize well across datasets with different homophily regimes.


\begin{figure}[h]
\centering
\includegraphics[width=\textwidth]{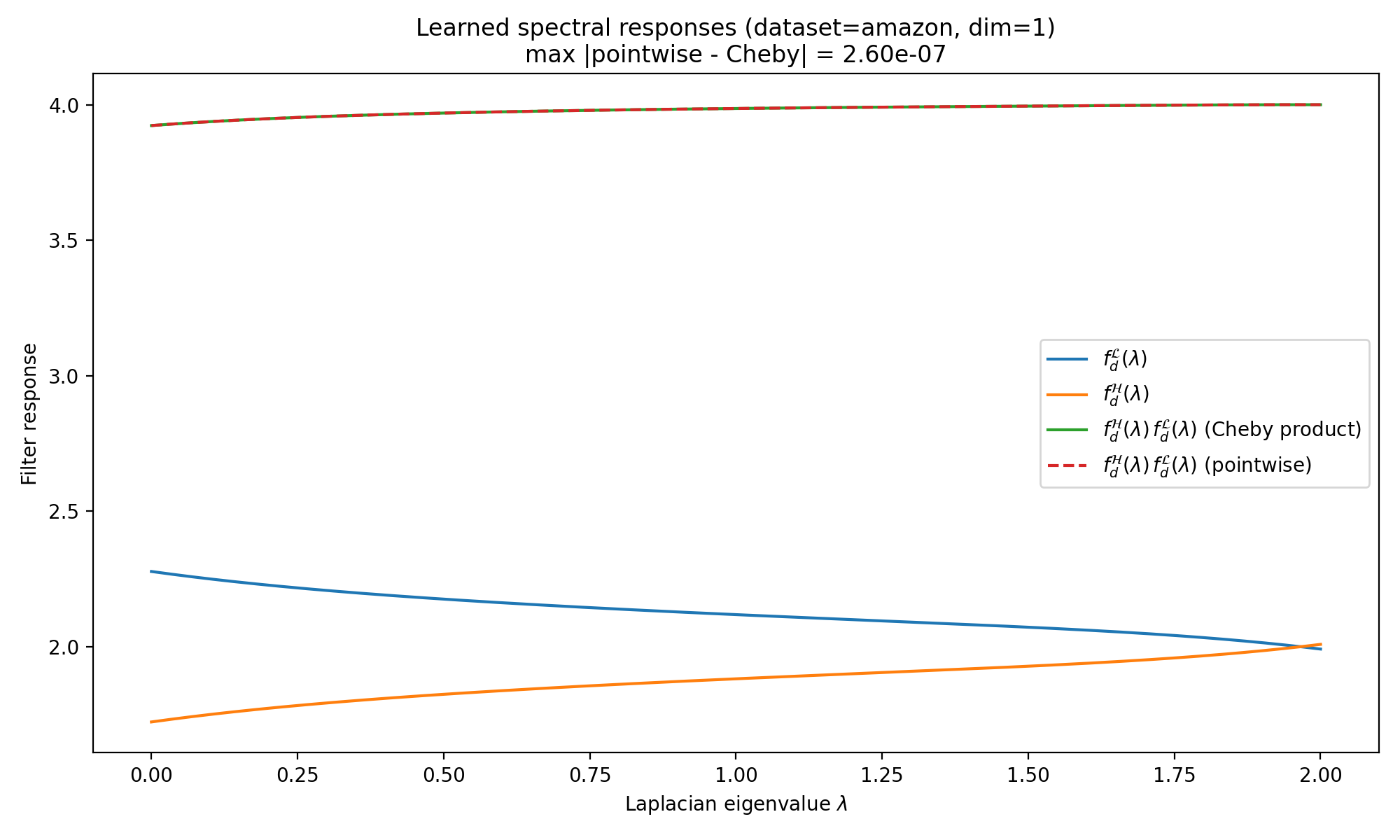}
\caption{Spectral frequency responses of the learned filters of \methodname\ on Amazon. For each dimension $d$, we plot the learned low-pass response $f_d^{\mathcal{L}}(\lambda)$, the learned high-pass response $f_d^{\mathcal{H}}(\lambda)$, and their composed response $f_d^{\mathcal{H}}(\lambda)f_d^{\mathcal{L}}(\lambda)$, evaluated on a dense grid of Laplacian eigenvalues $\lambda \in [0,2]$ (normalized Laplacian). We also overlay the pointwise product $f_d^{\mathcal{H}}(\lambda)f_d^{\mathcal{L}}(\lambda)$ to verify that it matches the Chebyshev-product implementation.}
\label{fig:filter_response_amazon}
\end{figure}

\begin{figure}[h]
\centering
\centering
\includegraphics[width=\textwidth]{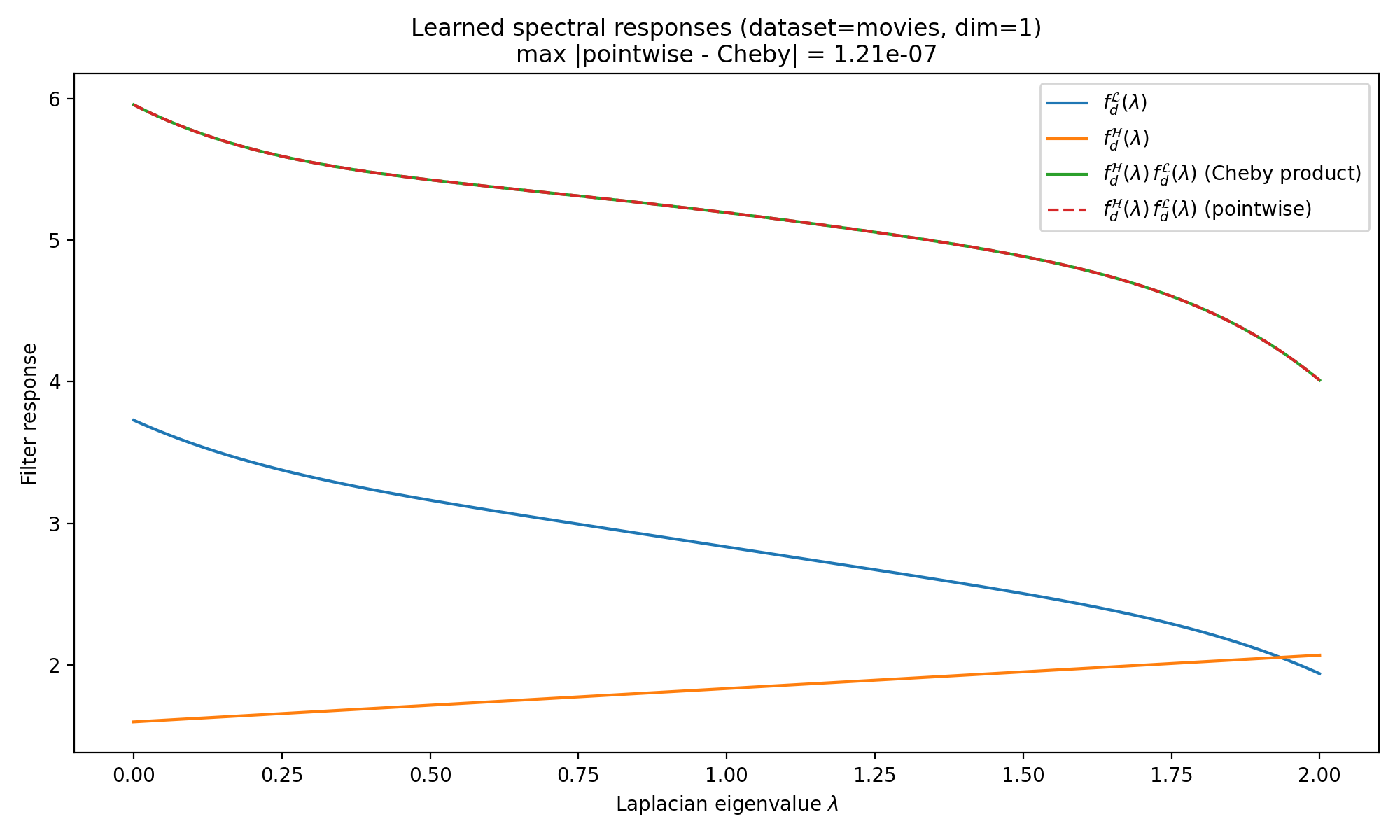}
\caption{Spectral frequency responses of the learned filters of \methodname\ on Movies, plotted with the same protocol as in Fig.~\ref{fig:filter_response_amazon}.}
\label{fig:filter_response_movies}
\end{figure}

\subsection{Spectral Filter Responses}
To provide an interpretable view of how \methodname\ adapts to homophily and heterophily at the signal level, we visualize the spectral frequency response of the learned Chebyshev filters for the multiplex dimensions. Concretely, after training we extract the learned Chebyshev coefficients and evaluate the low-pass filter
$f_d^{\mathcal{L}}(\lambda)$, the high-pass filter $f_d^{\mathcal{H}}(\lambda)$, and the composed filter $f_d^{\mathcal{H}}(\lambda)f_d^{\mathcal{L}}(\lambda)$ on a dense grid of eigenvalues $\lambda \in [0,2]$ corresponding to the spectrum of the normalized Laplacian. This provides an intuitive picture of which frequency components are attenuated or amplified by the model in each dimension.

Figs.~\ref{fig:filter_response_amazon} and \ref{fig:filter_response_movies} show the learned responses on Amazon and Movies. The dashed pointwise product closely overlaps the composed response obtained via the Chebyshev-product coefficients, which empirically validates the correctness of Prop.~\ref{prop:2}.

On Amazon, which exhibits low homophily ratios across dimensions (Table~\ref{table:data_statistics}), the learned low-pass and high-pass components show complementary behaviour, indicating that the model does not rely on purely homophilic smoothing. Importantly, the composed response remains stable across the
spectrum while still preserving dimension-specific differences, which is consistent with the gains obtained by the full model on Amazon in Table~\ref{table:node_classification} and with the ablation results showing that combining $H_d$, proximal consensus, and composed filtering yields the best overall performance (Table~\ref{table:ablation_study}).

On Movies, where the homophily levels are comparatively higher, the composed response places more emphasis on lower frequencies, reflecting a stronger contribution of smoothing components that are beneficial under homophily. At the same time, the presence of a non-trivial high-pass component indicates that the model still preserves heterophilic information when they are relevant. Overall, these frequency-response plots corroborate the main motivation of \methodname: learning dimension-specific filters that flexibly combine homophilic and heterophilic propagation patterns rather than committing to a single regime.

\begin{figure}[t]
\centering
\begin{subfigure}{0.48\linewidth}
    \centering
    \includegraphics[width=\textwidth]{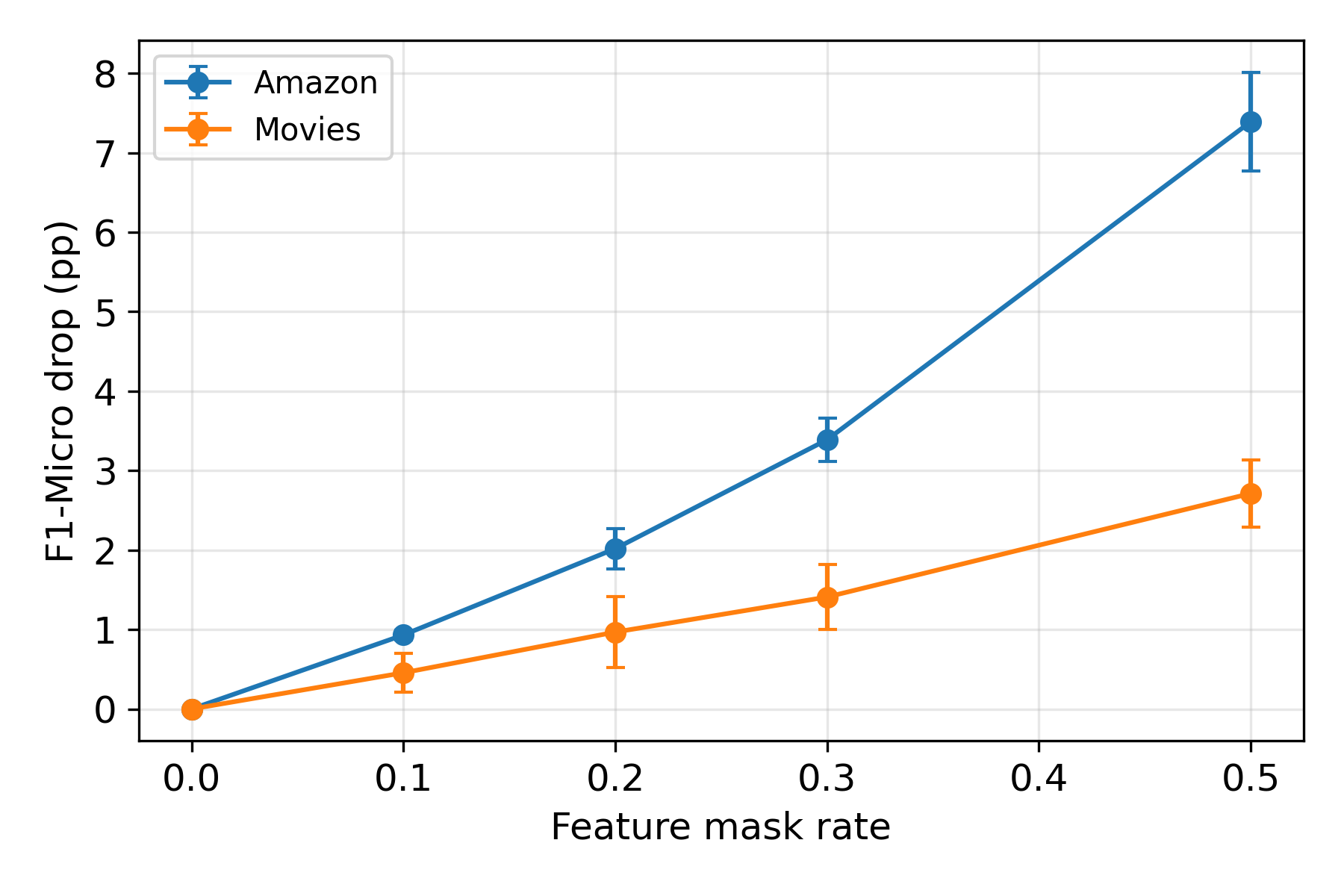}
    \caption{Feature masking.}\label{fig:robust_feat_mask}
\end{subfigure}
\begin{subfigure}{0.48\linewidth}
    \centering
    \includegraphics[width=\textwidth]{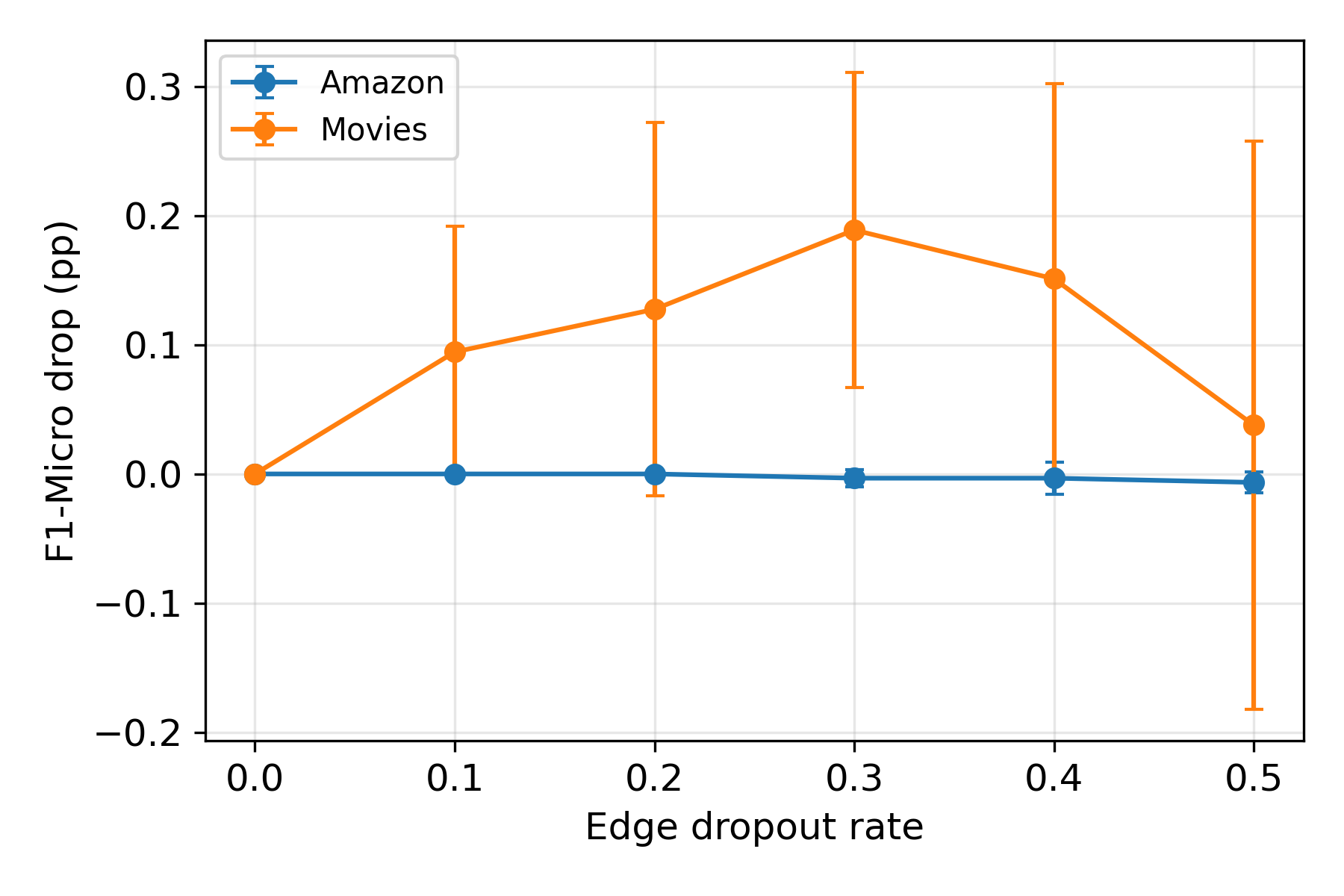}
    \caption{Edge dropping.}\label{fig:robust_edge_dropout}
\end{subfigure}

\caption{Robustness of \methodname\ under input perturbations on two real-world multiplex datasets (Amazon and Movies). The y-axis reports the drop in F1-Micro (in percentage points) relative to the clean graph. Error bars denote standard deviation over five random perturbation trials.}
\label{fig:robustness}
\end{figure}

\subsection{Robustness Analysis}
We evaluate the robustness of \methodname\ under two common perturbations that emulate noisy or incomplete multiplex data: \textit{(i) feature masking}, where a fraction $p$ of node feature entries is randomly set to zero, and \textit{(ii) edge dropping}, where a fraction $p$ of edges is randomly removed in each dimension. We apply perturbations only at test time (the model is trained on the clean graph) and report the corresponding performance drop relative to the clean setting. Each perturbation level is repeated over five random trials and we report mean and standard deviation (Fig.~\ref{fig:robustness}).

\textit{Feature masking.} Fig.~\ref{fig:robust_feat_mask} shows that \methodname\ degrades gracefully as feature information is removed. With $50\%$ masking, the F1-Micro drop is $7.39$ pp on Amazon and $2.71$ pp on Movies (F1-Macro drop: $7.48$ pp and $2.15$ pp, respectively). This indicates that \methodname\ can compensate for partially missing attributes by leveraging multiplex structural signals captured by the composed low-/high-pass filters and the consensus mechanism.

\textit{Edge dropping.} Fig.~\ref{fig:robust_edge_dropout} shows that \methodname\ is highly stable to substantial edge removal. Across dropout rates up to $50\%$, the maximum observed F1-Micro drop remains below $0.19$ pp on Movies and below $0.01$ pp on Amazon (the latter even exhibits a negligible gain within noise). This robustness is consistent with the design of \methodname, which integrates information across multiple relations and relies on learned compatibility matrices to reweight same-class and cross-class propagation rather than overfitting to a specific connectivity pattern. Overall, these stress tests suggest that \methodname\ remains reliable under common forms of missing features and missing links in multiplex graphs.

\begin{figure}[t]
\centering
\begin{subfigure}{0.4\linewidth}
    \centering
    \includegraphics[width=\textwidth]{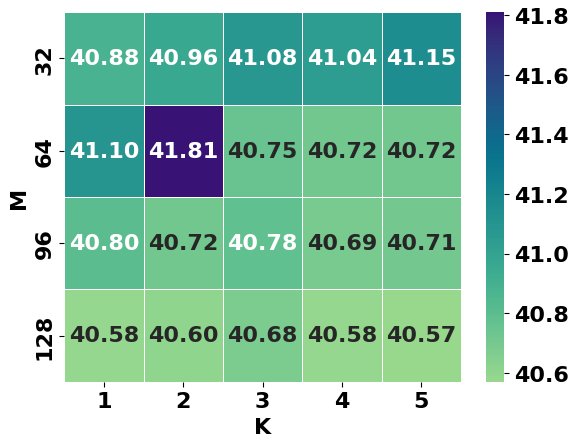}
    \caption{F1-Macro on Movies}\label{fig:sensitivity_movies_macro}
\end{subfigure}
\begin{subfigure}{0.4\linewidth}
    \centering
    \includegraphics[width=\textwidth]{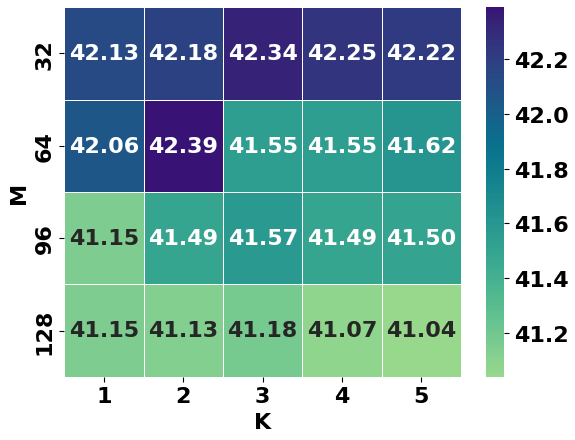}
    \caption{F1-Micro on Movies}\label{fig:sensitivity_movies_micro}
\end{subfigure}

\begin{subfigure}{0.4\linewidth}
    \centering
    \includegraphics[width=\textwidth]{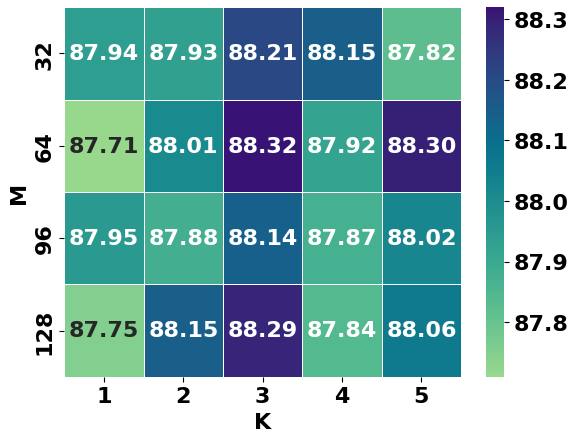}
    \caption{F1-Macro on Amazon}\label{fig:sensitivity_amazon_macro}
\end{subfigure}
\begin{subfigure}{0.4\linewidth}
    \centering
    \includegraphics[width=\textwidth]{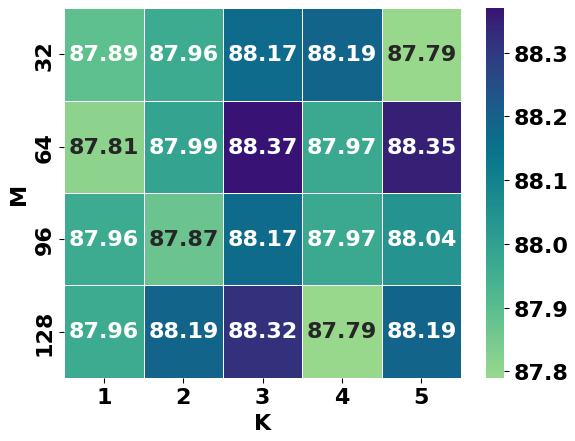}
    \caption{F1-Micro on Amazon}\label{fig:sensitivity_amazon_micro}
\end{subfigure}

\caption{Sensitivity analysis of \methodname.}
\label{fig:sensitivity_analysis}
\end{figure}

\subsection{Sensitivity Analysis}
Fig. \ref{fig:sensitivity_analysis} presents the sensitivity analysis of \methodname, showing the impact of varying the embedding size $M$ and the degree of the filters $K$ on the performance of node classification. The heatmaps reveal that our approach is relatively stable across different settings of $M$ and $K$, with only minor fluctuations in accuracy. For the Movies dataset, the scores show a slight preference for a smaller number of filters ($K=$1 or 2), particularly when the embedding size is set to $M=64$. In contrast, the best performance for the Amazon dataset is observed at a larger degree ($K=$3 or 5). Overall, the analysis suggests that the model performs robustly across a wide range of settings. 

\begin{figure}[t]
\centering
\begin{subfigure}{0.98\linewidth}
    \centering
    \includegraphics[width=\textwidth]{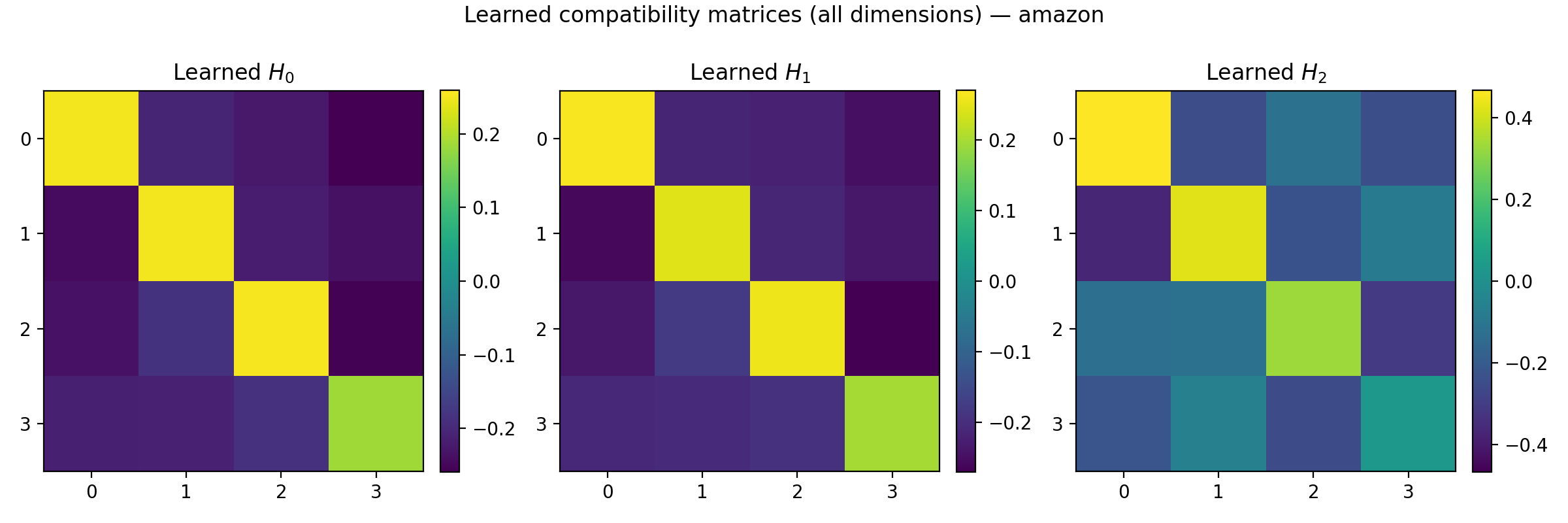}
    \caption{Amazon (three dimensions).}\label{fig:H_learned_amazon}
\end{subfigure}

\vspace{0.5em}

\begin{subfigure}{0.98\linewidth}
    \centering
    \includegraphics[width=\textwidth]{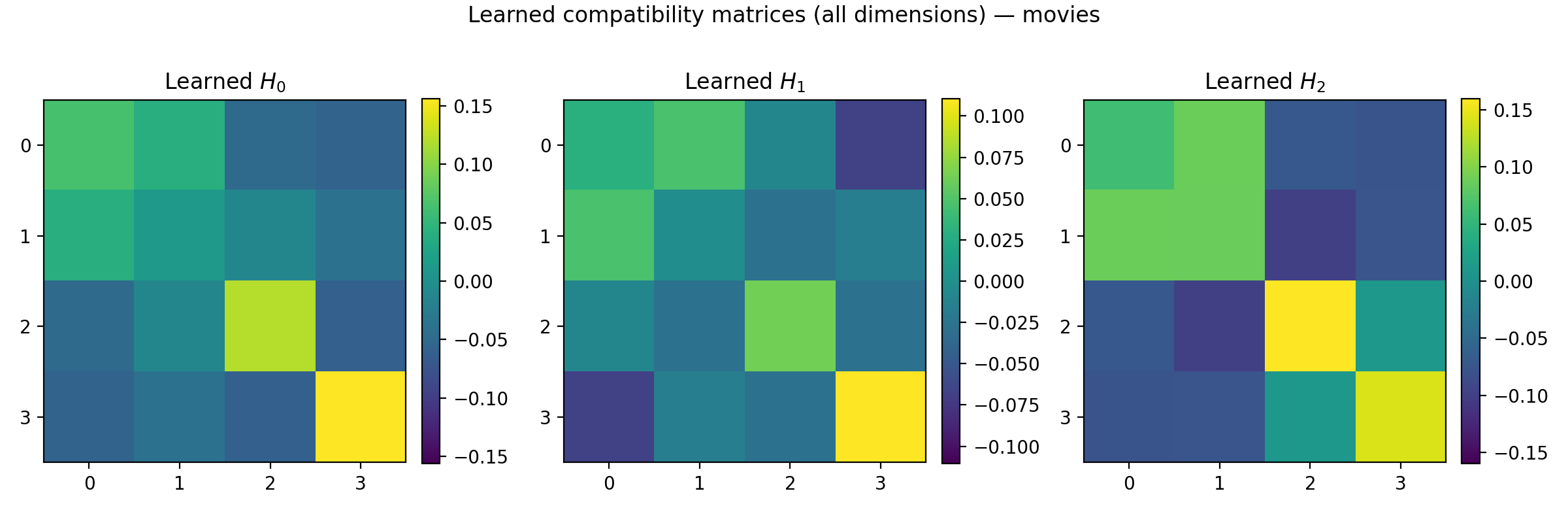}
    \caption{Movies (three dimensions).}\label{fig:H_learned_movies}
\end{subfigure}

\caption{Learned dimension-specific compatibility matrices $H_d$ of \methodname\ on two datasets (Amazon and Movies). Each heatmap is a $C\times C$ matrix, where entry $(H_d)_{c,c'}$ controls how information associated with class $c$ contributes to the score of class $c'$ after propagation in dimension $d$. Diagonal-dominant structure corresponds to homophilic couplings, whereas strong off-diagonal structure (in magnitude) indicates cross-class (heterophilic) couplings.}
\label{fig:compatibility_matrices}
\end{figure}

\subsection{Learned Compatibility Matrices}
A key component of \methodname\ is the dimension-specific compatibility matrix $H_d \in \mathbb{R}^{C \times C}$, which is learned jointly with the spectral filters and the
consensus mechanism. Fig.~\ref{fig:compatibility_matrices} provides a qualitative view of the learned $H_d$ matrices for Amazon and Movies. First, the learned compatibility patterns differ substantially across dimensions, supporting our design choice to learn dimension-specific $H_d$ rather than a single shared compatibility matrix. Second, the matrices exhibit both diagonal structure and pronounced off-diagonal structure (with dataset- and dimension-dependent intensity), indicating that \methodname\ does not assume a fixed homophily regime but instead learns how to reweight same-class and cross-class signals in a relation-aware manner.

These visual results complement the quantitative ablation findings in Table~\ref{table:ablation_study}. More precisely, introducing $H_d$ on top of the naive per-dimension GCN yields consistent gains, with a particularly large improvement on Amazon (F1-Macro: $74.86 \rightarrow 86.62$). Thus, the compatibility modeling is a major contributor to the overall performance improvements reported in Table~\ref{table:node_classification}.

\begin{figure}[!ht]
\centering
\begin{subfigure}{0.48\linewidth}
    \centering
    \includegraphics[width=\textwidth]{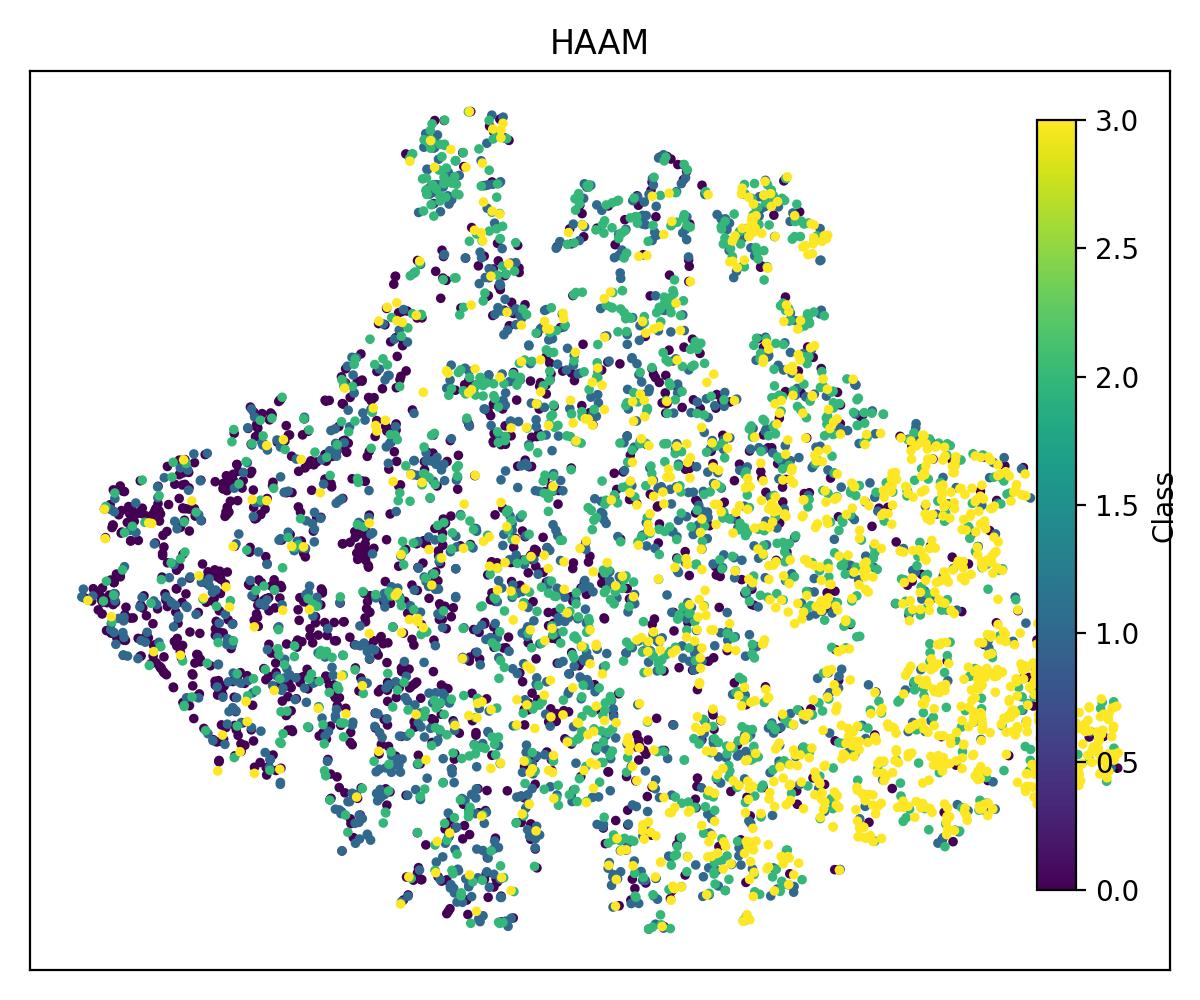}
    \caption{\methodname\ on Movies.}\label{fig:tsne_movies_haam}
\end{subfigure}
\begin{subfigure}{0.48\linewidth}
    \centering
    \includegraphics[width=\textwidth]{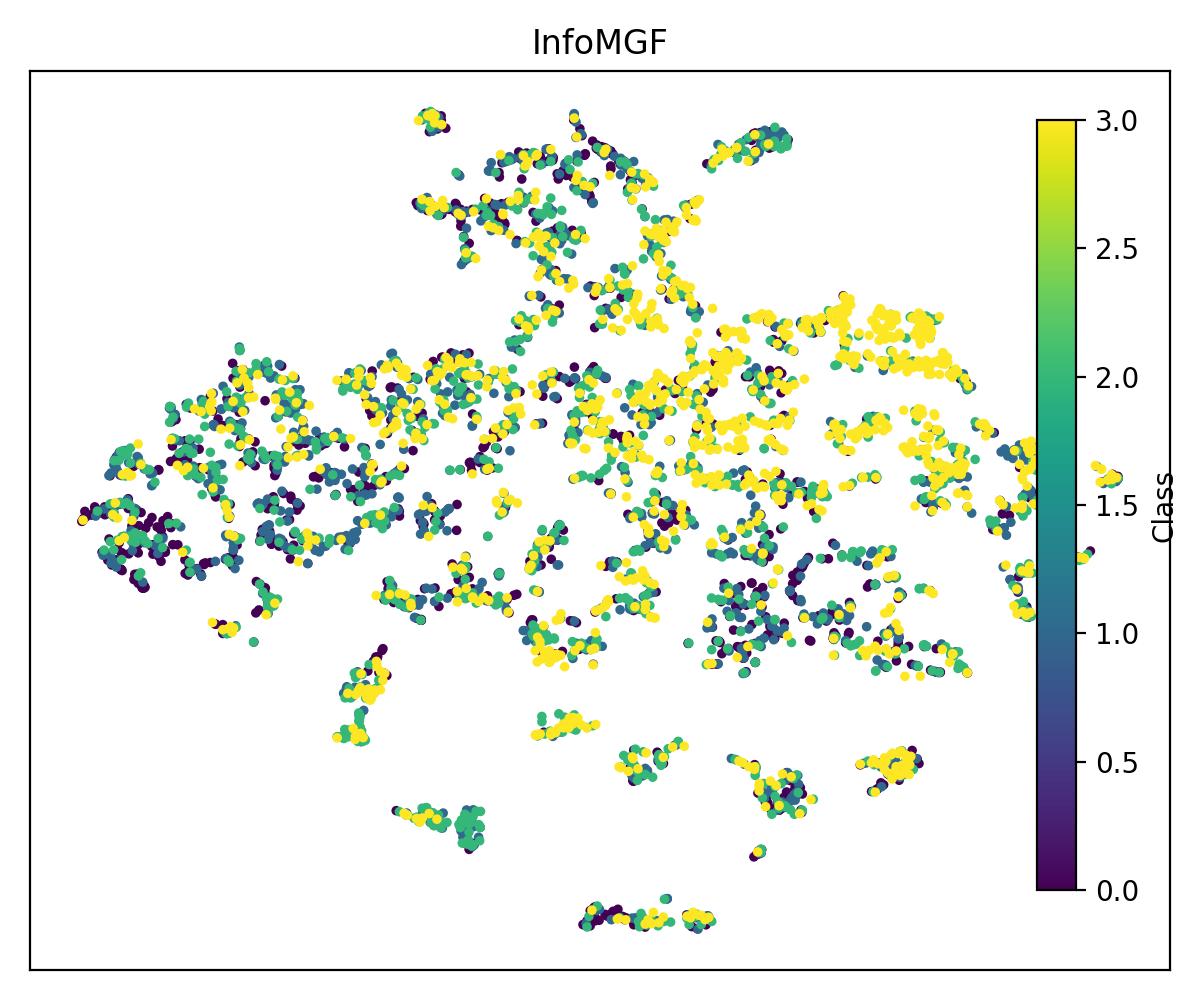}
    \caption{InfoMGF on Movies.}\label{fig:tsne_movies_infomgf}
\end{subfigure}

\begin{subfigure}{0.48\linewidth}
    \centering
    \includegraphics[width=\textwidth]{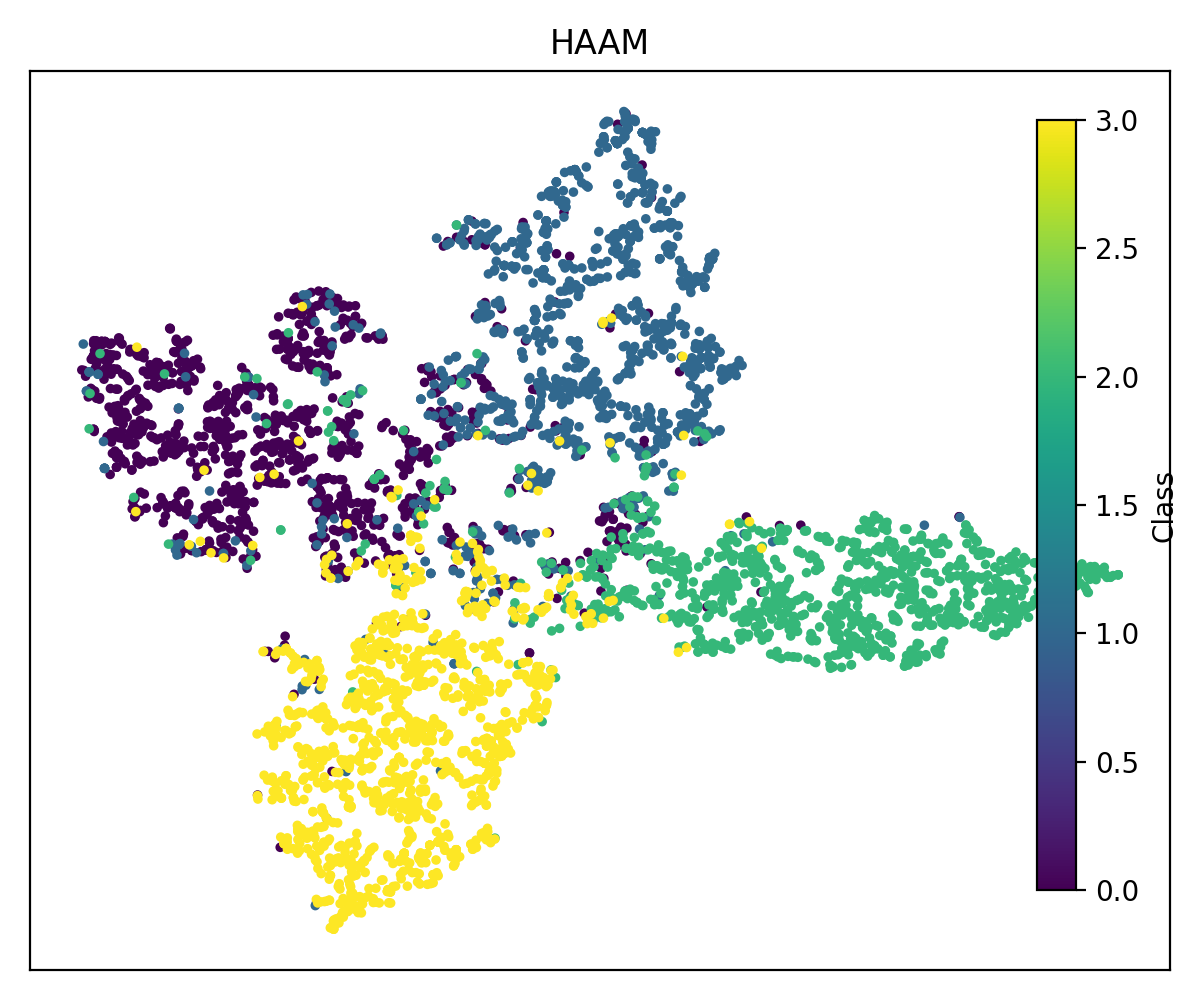}
    \caption{\methodname\ on Amazon.}\label{fig:tsne_amazon_haam}
\end{subfigure}
\begin{subfigure}{0.48\linewidth}
    \centering
    \includegraphics[width=\textwidth]{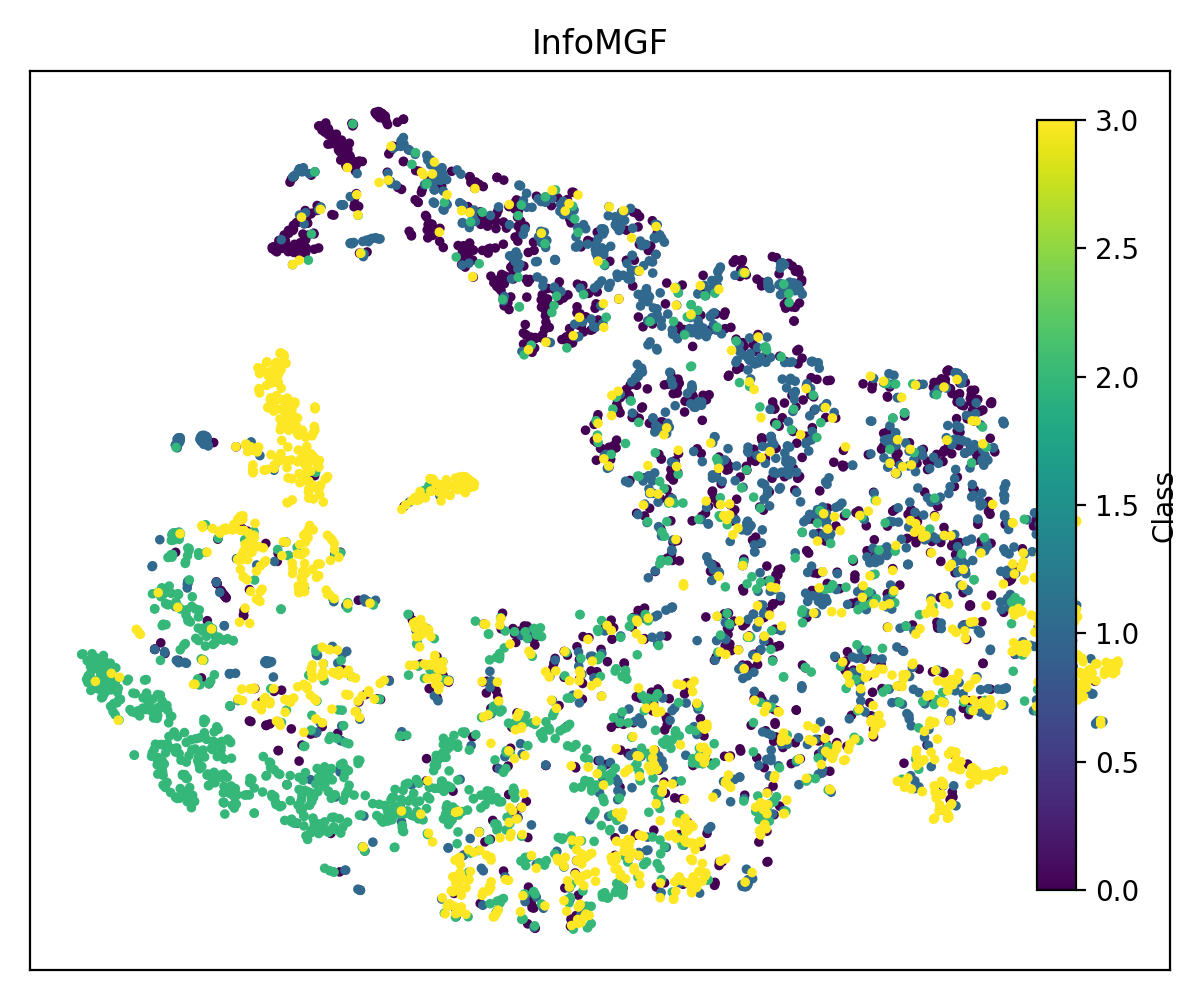}
    \caption{InfoMGF on Amazon.}\label{fig:tsne_amazon_infomgf}
\end{subfigure}

\caption{Latent space t-SNE visualization of the node representations learned by \methodname\ and InfoMGF on Movies and Amazon. Points are colored by ground-truth classes. Following common practice, we subsample up to $5{,}000$ nodes (approximately balanced across classes) and use a fixed random seed.}
\label{fig:tsne_visualization}
\end{figure}

\subsection{Embedding Visualization}
We provide a qualitative assessment of representation quality using t-SNE \cite{maaten2008visualizing}. The node representations produced by each method are extracted after training, project them to two dimensions, and color nodes by their ground-truth class labels. As illustrated in Fig.~\ref{fig:tsne_visualization}, \methodname\ yields representations that are more class-separable. On Amazon, the embedding space of \methodname\ exhibits compact and well-separated clusters with limited overlap between categories. In contrast, InfoMGF shows a more fragmented and mixed structure, which is consistent with its lower F1 scores on this dataset (Table~\ref{table:node_classification}). On Movies, class overlap remains more pronounced for both methods. Nevertheless, the \methodname\ embeddings appear less fragmented, which aligns with its consistent improvement in F1 over competing approaches.

\section{Conclusion}

This paper addresses a critical gap in the literature of multiplex graphs by introducing \methodname, a novel adaptive framework specifically designed for node classification in multiplex graphs exhibiting both heterophilic and homophilic dimensions. While existing models tend to focus on homophilic structures, they may fall short in capturing the structural diversity of real-world systems where heterophily is prevalent. Our method leverages the product of high-pass and low-pass Chebyshev filters, combined with dimension-specific compatibility matrices and consensus labels, to dynamically adapt to varying connectivity patterns across dimensions. This design supports the ability to generate more stable and robust predictions in multi-relational environments. 

The dimension-specific compatibility matrices are estimated using labeled nodes within the semi-supervised setting. As a result, their reliability depends on the availability and representativeness of labeled data. When labeled nodes are extremely scarce, highly imbalanced, or biased, the empirical class-mixing estimates may become less stable, which can in turn affect calibration and prediction robustness. Although this behavior is inherent to approaches relying on class-conditional connectivity statistics, incorporating additional regularization or prior structural assumptions could further improve robustness in low-label regimes. Exploring such strategies constitutes a promising direction for future work.

While this study focuses exclusively on multiplex graphs, in which nodes of the same type are connected via multiple types of edges, it does not target more general heterogeneous multidimensional graphs that involve both multiple types of nodes and multiple types of edges. Likewise, it does not address dynamic graphs that evolve over time. Nonetheless, this work may pave the way for exploring more complex graph structures, including heterogeneous multiplex graphs and dynamic multiplex networks.

\section*{Declaration of Competing Interest}
The authors declare that they have no known competing financial interests or personal relationships that could have appeared to influence the work reported in this paper.

\section*{Data Availability}
The dataset and the source code used to support the findings of this study are available at
\href{https://drive.google.com/drive/folders/1ROhUghYARMRGzyLyBP2wXJH6MxXfu4yf?usp=drive_link}{this link}.

\bibliographystyle{plainnat}
\bibliography{references}     

\newpage 

\appendix
\section*{Appendices}

\section{Notation summary}\label{sec:notation_summary}

In the main manuscript, we use a variety of symbols to denote graph structure, node- and label-related quantities, and the parameters of the spectral filters. For convenience, this appendix gathers the used notation in a single place. The goal is to provide a quick reference while reading Sections 3 and 4, particularly the derivations of the Chebyshev filters and the update rules.

Table~\ref{tab:notation_summary} lists each symbol, along with a brief description and its dimensions. This helps distinguish, for example, between node-level matrices ($X$, $Y$, $\hat{Y}_d$, $\hat{Y}$), graph-structural operators ($A_d$, $L_d$, $\hat{L}_d$), and filter parameters ($\theta_k^{\mathcal{L},d}$, $\theta_k^{\mathcal{H},d}$, $\gamma^{\mathcal{L},d}$, $\gamma^{\mathcal{H},d}$). Dataset-specific notation is introduced directly in 
Sec.~5 when needed.

\begin{table}[h]
\centering
\caption{Summary of main notation.}
\label{tab:notation_summary}
\begin{tabular}{lll}
\hline
Symbol & Description & Dimensions \\
\hline
$N$ & Number of nodes & $\mathbb{N}$ \\
$D$ & Number of dimensions (layers) & $\mathbb{N}$ \\
$C$ & Number of classes & $\mathbb{N}$ \\
$F$ & Number of input node features & $\mathbb{N}$ \\
$F'$ & Dimension of transformed features & $\mathbb{N}$ \\
$K$ & Degree of polynomial filters & $\mathbb{N}$ \\
$X$ & Node feature matrix & $\mathbb{R}^{N \times F}$ \\
$Y$ & One-hot label matrix & $\{0,1\}^{N \times C}$ \\
$A_d$ & Adjacency matrix in dimension $d$ & $\mathbb{R}^{N \times N}$ \\
$\Delta_d$ & Degree matrix in dimension $d$ & $\mathbb{R}^{N \times N}$ \\
$L_d$ & Normalized Laplacian in dimension $d$ & $\mathbb{R}^{N \times N}$ \\
$\tilde{L}_d$ & Rescaled Laplacian in dimension $d$ & $\mathbb{R}^{N \times N}$ \\
$\hat{L}_d$ & Product spectral filter in dimension $d$ & $\mathbb{R}^{N \times N}$ \\
$H_d$ & Compatibility matrix in dimension $d$ & $\mathbb{R}^{C \times C}$ \\
$\hat{Y}_0$ & Initial label predictions (MLP output) & $\mathbb{R}^{N \times C}$ \\
$\hat{Y}_d$ & Dimension-specific label predictions & $\mathbb{R}^{N \times C}$ \\
$\hat{Y}$ & Consensus label predictions & $\mathbb{R}^{N \times C}$ \\
$f_d^{\mathcal{L}}$, $f_d^{\mathcal{H}}$ & Low-/high-pass Chebyshev filters & $\mathbb{R}^{N \times N}$ (operators) \\
$\theta_k^{\mathcal{L},d}$, $\theta_k^{\mathcal{H},d}$ & Chebyshev coefficients & Scalars, $k\in \left\{0,\dots,K\right\}$ \\
$\gamma^{\mathcal{L},d}$, $\gamma^{\mathcal{H},d}$ & Reparameterization vectors & $\mathbb{R}^{K+1}$ \\
\hline
\end{tabular}
\end{table}

\section{Proof of Proposition \ref{prop.A}}\label{proof:prop1}

\begin{proof}

Let $L$ be the graph Laplacian with eigenvalue decomposition $L = U \Lambda U^{\top}$, where $\Lambda = \mathrm{diag}(\lambda_1, \dots, \lambda_N)$ is the diagonal matrix of eigenvalues and $U$ is the matrix of corresponding eigenvectors. The matrix of eigenvectors $U$ is an orthonormal matrix, then $U^{\top} \, U = I$. The graph signal $x$ can be transformed into the spectral domain using the graph Fourier transform:
\[
\hat{x} = U^{\top} x.
\]

\noindent Now, applying a low-pass filter $f^{\mathcal{L}}(L)$ to the signal $x$ results in:
\[
f^{\mathcal{L}}(L) \, x = U \, f^{\mathcal{L}}(\Lambda) \, U^{\top} x = U \, f^{\mathcal{L}}(\Lambda) \, \hat{x},
\]
where $f^{\mathcal{L}}(\Lambda)$ is the element-wise application of the low-pass filter to the eigenvalues in $\Lambda$. Similarly, applying a high-pass filter $f^{\mathcal{H}}(L)$ after the low-pass filter gives:

\begin{equation*}
\begin{split}
    f^{\mathcal{H}}(L) \, \left( f^{\mathcal{L}}(L) \, x \right) &= U \, f^{\mathcal{H}}(\Lambda) \, U^{\top} \, U \, f^{\mathcal{L}}(\Lambda) \, \hat{x} \\
    &= U \, f^{\mathcal{H}}(\Lambda) \, f^{\mathcal{L}}(\Lambda) \, \hat{x}.
\end{split}
\end{equation*}

\noindent Thus, the combined filter $f(L)$ is given by the product of $f^{\mathcal{H}}(\Lambda)$ and $f^{\mathcal{L}}(\Lambda)$:
\[
f(L) = f^{\mathcal{H}}(L) \cdot f^{\mathcal{L}}(L),
\]
and the filter output becomes:
\[
y = U \, \left( f^{\mathcal{H}}(\Lambda) \cdot f^{\mathcal{L}}(\Lambda) \right) \, U^{\top} x.
\]
\end{proof}

\section{Proof of Corollary \ref{coro.1}}\label{proof:coro1}

\begin{proof}
According to Prop. 4.1, the application of a low-pass filter $f^{\mathcal{L}}(L)$ followed by a high-pass filter $f^{\mathcal{H}}(L)$ to a graph signal $x$ is equivalent to applying a filter whose eigenvalues are the element-wise product of the eigenvalues of the low-pass and high-pass filters. 

\vspace{1mm}
\noindent So, we have:
\[
y = f(L) \, x = U \left( f^{\mathcal{H}}(\Lambda) \cdot f^{\mathcal{L}}(\Lambda) \right) U^{\top} x.
\]

\noindent Since the element-wise product of diagonal matrices (eigenvalue matrices) is equivalent to the dot product of the matrices, we can conclude that the dot product of diagonal matrices is commutative:
\[
f^{\mathcal{H}}(\Lambda) \cdot f^{\mathcal{L}}(\Lambda) = f^{\mathcal{L}}(\Lambda) \cdot f^{\mathcal{H}}(\Lambda),
\]
it then follows that:
\[
f^{\mathcal{H}}(L) \cdot f^{\mathcal{L}}(L) = f^{\mathcal{L}}(L) \cdot f^{\mathcal{H}}(L).
\]

\noindent Thus, the order of applying the low-pass and high-pass filters does not affect the outcome, proving that the composition is order-invariant.

\end{proof}

\section{Proof of Proposition \ref{prop:2}}\label{proof:prop2}

We first show the following result, which is used to prove the proposition.
\begin{lemma} \label{lemma1}
Let $T_i(x)$ and $T_j(x)$ be Chebyshev polynomials of degree $i$ and $j$ respectively. Then, their product can be expressed as:
\begin{equation*}
    T_i(x) \cdot T_j(x) = \frac{1}{2} \left[T_{i+j}\left(x\right) + T_{|i-j|}\left(x\right)\right].
\end{equation*}
\end{lemma}

\begin{proof}
    The equality can be proved using the trigonometric definition of Chebyshev polynomials:
    \begin{equation*}
        T_k(x) = \cos(k \arccos(x)),
    \end{equation*}
    and the product-to-sum formula for the cosine:
    \begin{equation*}
        \cos(\alpha) \cdot \cos(\beta) = \frac{1}{2} \left[\cos(\alpha + \beta) + \cos(\alpha - \beta)\right].
    \end{equation*}
    Applying both formulas to $T_i(x) \cdot T_j(x)$ results in:
    \begin{align*}
        T_i(x) \cdot T_j(x) &= \cos(i \arccos(x)) \cdot \cos(j \arccos(x)) \\
        &= \frac{1}{2} [\cos\left(\left(i+j\right) \cdot\arccos x\right) \\
        & \ \ \ + \cos\left(\left(i - j\right) \cdot \arccos x\right)] \\
        &= \frac{1}{2} \left[T_{i+j}\left(x\right) + T_{|i-j|}\left(x\right)\right].
    \end{align*}
\end{proof}

We can now prove Prop.~\ref{prop:2}
\begin{proof}
    Let $\hat{L}_d = f_d^{\mathcal{L}}(\tilde{L}_d) \cdot f_d^{\mathcal{H}}(\tilde{L}_d)$. The product of $f_d^{\mathcal{L}}$ and $f_d^{\mathcal{H}}$ can be expressed as:
    \begin{align*}
        f_d^{\mathcal{L}}(\tilde{L}_d) \cdot f_d^{\mathcal{H}}(\tilde{L}_d) &= \left(\sum_{k=0}^K \theta_k^{\mathcal{L},d} \: T_k(\tilde{L}_d)\right) \left(\sum_{k=0}^K \theta_k^{\mathcal{H},d} \: T_k(\tilde{L}_d)\right)  \\
        &= \sum_{i=0}^K \sum_{j=0}^K \theta_i^{\mathcal{L},d} \, \theta_j^{\mathcal{H},d} \, T_i(\tilde{L}_d) \, \cdot \, T_j(\tilde{L}_d).
    \end{align*}

    We apply Lemma \ref{lemma1} to express $T_i(\tilde{L}_d) \cdot T_j(\tilde{L}_d)$ in terms of the sum of $T_{i+j}\left(\tilde{L}_d\right)$ and $T_{|i-j|}\left(\tilde{L}_d\right)$. Thus, the double sum can be rewritten as:
    \begin{equation*}
        \hat{L_d} = \frac{1}{2} \sum_{i=0}^K \sum_{j=0}^K \theta_i^{\mathcal{L},d} \theta_j^{\mathcal{H},d} \left[T_{i+j}\left(\tilde{L}_d\right) + T_{|i-j|}\left(\tilde{L}_d\right)\right].
    \end{equation*}
\end{proof}


\section{Proof of Proposition~\ref{prop:cheb_basis_bound}}
\label{proof:prop_cheb_basis_bound}

\begin{proof}
Since $\tilde{L}_d$ is symmetric, it admits an eigendecomposition $\tilde{L}_d = U_d \tilde{\Lambda}_d U_d^\top$, where $U_d$ is orthonormal and $\tilde{\Lambda}_d=\mathrm{diag}(\tilde{\lambda}_{d}^{(1)},\dots,\tilde{\lambda}_{d}^{(N)})$ with $\tilde{\lambda}_{d}^{(i)}\in[-1,1]$ for all $i$. Because $T_k(\cdot)$ is a polynomial, we have:
\[
T_k(\tilde{L}_d) = U_d \, T_k(\tilde{\Lambda}_d)\, U_d^\top,
\qquad
T_k(\tilde{\Lambda}_d)=\mathrm{diag}\!\big(T_k(\tilde{\lambda}_{d}^{(1)}),\dots,T_k(\tilde{\lambda}_{d}^{(N)})\big).
\]
Therefore,
\[
\|T_k(\tilde{L}_d)\|_2
=
\|T_k(\tilde{\Lambda}_d)\|_2
=
\max_{1\le i\le N} \big|T_k(\tilde{\lambda}_{d}^{(i)})\big|.
\]
Since $|T_k(x)|\le 1$ for all $x\in[-1,1]$, we obtain $\|T_k(\tilde{L}_d)\|_2\le 1$.
\end{proof}

\section{Proof of Proposition~\ref{prop:bibo_cheb}}
\label{proof:prop_bibo_cheb}

\begin{proof}
Let $f_d(\tilde{L}_d)=\sum_{k=0}^{K}\alpha_k\,T_k(\tilde{L}_d)$. By the triangle inequality and Prop.~\ref{prop:cheb_basis_bound}, we have:
\begin{align*}
\|f_d(\tilde{L}_d)\|_2
&\le \sum_{k=0}^{K} |\alpha_k|\,\|T_k(\tilde{L}_d)\|_2
\le \sum_{k=0}^{K} |\alpha_k|.
\end{align*}

For any matrix-valued signal $S\in\mathbb{R}^{N\times C}$, we similarly obtain:
\begin{align*}
\|f_d(\tilde{L}_d)\,S\|_F
&= \Big\|\sum_{k=0}^{K}\alpha_k\,T_k(\tilde{L}_d)\,S\Big\|_F \\
&\le \sum_{k=0}^{K} |\alpha_k|\,\|T_k(\tilde{L}_d)\,S\|_F \\
&\le \sum_{k=0}^{K} |\alpha_k|\,\|T_k(\tilde{L}_d)\|_2\,\|S\|_F \\
&\le \Big(\sum_{k=0}^{K}|\alpha_k|\Big)\,\|S\|_F,
\end{align*}
again using Prop.~\ref{prop:cheb_basis_bound}.
\end{proof}

\section{Proof of Corollary~\ref{coro:product_stability}}
\label{proof:coro_product_stability}

\begin{proof}
By submultiplicativity of the spectral norm, we have:
\begin{align*}
\|\hat{L}_d\|_2
&=
\|f_d^{\mathcal{L}}(\tilde{L}_d)\,f_d^{\mathcal{H}}(\tilde{L}_d)\|_2 \\
&\le
\|f_d^{\mathcal{L}}(\tilde{L}_d)\|_2\;\|f_d^{\mathcal{H}}(\tilde{L}_d)\|_2.
\end{align*}
Applying Prop.~\ref{prop:bibo_cheb} to each factor yields:
\[
\|f_d^{\mathcal{L}}(\tilde{L}_d)\|_2 \le \sum_{k=0}^{K}|\theta_k^{\mathcal{L},d}|,
\qquad
\|f_d^{\mathcal{H}}(\tilde{L}_d)\|_2 \le \sum_{k=0}^{K}|\theta_k^{\mathcal{H},d}|,
\]
which proves the stated bound on $\|\hat{L}_d\|_2$.

For any $S\in\mathbb{R}^{N\times C}$, we use $\|\hat{L}_d\,S\|_F \le \|\hat{L}_d\|_2\,\|S\|_F$ to obtain the Frobenius inequality in the corollary.

For the coefficient bound, Prop.~\ref{prop:2} expresses $\hat{L}_d$ as a degree-$2K$ Chebyshev expansion with coefficients obtained by summing contributions of the form $\frac{1}{2}\theta_i^{\mathcal{L},d}\theta_j^{\mathcal{H},d}$ (each pair $(i,j)$ contributes to at most two terms: $T_{i+j}$ and $T_{|i-j|}$). By the triangle inequality, the $\ell_1$ norm of the resulting coefficient vector
$\bar{\theta}^{\,d}$ satisfies:
\begin{align*}
\|\bar{\theta}^{\,d}\|_1
&\le
\frac{1}{2}\sum_{i=0}^{K}\sum_{j=0}^{K}\big|\theta_i^{\mathcal{L},d}\theta_j^{\mathcal{H},d}\big|
+
\frac{1}{2}\sum_{i=0}^{K}\sum_{j=0}^{K}\big|\theta_i^{\mathcal{L},d}\theta_j^{\mathcal{H},d}\big| \\
&=
\Big(\sum_{i=0}^{K}|\theta_i^{\mathcal{L},d}|\Big)\,
\Big(\sum_{j=0}^{K}|\theta_j^{\mathcal{H},d}|\Big)
=
\|\theta^{\mathcal{L},d}\|_1\,\|\theta^{\mathcal{H},d}\|_1.
\end{align*}
\end{proof}

\section{Proof of Proposition~\ref{prop:softmax_stability}}
\label{proof:prop_softmax_stability}

\begin{proof}
We first prove the Frobenius bound. For any row vector $s\in\mathbb{R}^{C}$, let $p=\mathrm{softmax}(s)\in\mathbb{R}^{C}$, i.e., $p_c = \exp(s_c)/\sum_{c'=1}^C \exp(s_{c'})$.
Then $p_c\ge 0$ and $\sum_{c=1}^C p_c = 1$, hence $\|p\|_2 \le \|p\|_1 = 1$. Applying this row-wise to $Y=\mathrm{softmax}(S)$ gives:
\[
\|Y\|_F^2 = \sum_{i=1}^{N}\|Y_{i:}\|_2^2 \le \sum_{i=1}^{N} 1 = N,
\]
so $\|Y\|_F \le \sqrt{N}$.

We now prove the Lipschitz bound. Consider the vector softmax map
$\varphi:\mathbb{R}^{C}\to\mathbb{R}^{C}$, $\varphi(s)=\mathrm{softmax}(s)$.
Its Jacobian at $s$ is:
\[
J(s)=\nabla \varphi(s)=\mathrm{diag}(p)-p\,p^\top,
\qquad p=\varphi(s).
\]
For any unit vector $v\in\mathbb{R}^C$ with $\|v\|_2=1$, we have
\[
v^\top J(s)\,v
=
\sum_{c=1}^C p_c v_c^2
-
\Big(\sum_{c=1}^C p_c v_c\Big)^2
=
\mathrm{Var}(V),
\]
where $V$ is a real-valued random variable taking value $v_c$ with probability $p_c$.
Since $V\in[v_{\min},v_{\max}]$, Popoviciu's inequality yields:
\[
\mathrm{Var}(V)\le \frac{(v_{\max}-v_{\min})^2}{4}.
\]
Moreover, $v_{\max}-v_{\min} \le \max_{i,j}|v_i-v_j| \le \sqrt{2}\,\|v\|_2=\sqrt{2}$, because
$|v_i-v_j| = |(e_i-e_j)^\top v|\le \|e_i-e_j\|_2\,\|v\|_2=\sqrt{2}\,\|v\|_2$.
Therefore,
\[
v^\top J(s)\,v \le \frac{(\sqrt{2})^2}{4}=\frac{1}{2}
\qquad \forall v:\|v\|_2=1,
\]
which implies $\|J(s)\|_2\le 1/2$ for all $s\in\mathbb{R}^C$.
By the mean value theorem, for any $s,s'\in\mathbb{R}^C$, we obtain:
\[
\|\varphi(s)-\varphi(s')\|_2 \le \sup_{\tau\in[0,1]}\|J(s'+\tau(s-s'))\|_2\,\|s-s'\|_2
\le \frac{1}{2}\,\|s-s'\|_2.
\]

Applying this row-wise to $Y=\mathrm{softmax}(S)$ and $Y'=\mathrm{softmax}(S')$ yields:
\begin{align*}
\|Y-Y'\|_F^2
&=
\sum_{i=1}^N \|\varphi(S_{i:})-\varphi(S'_{i:})\|_2^2 \\
&\le
\sum_{i=1}^N \Big(\frac{1}{2}\|S_{i:}-S'_{i:}\|_2\Big)^2
=
\frac{1}{4}\,\|S-S'\|_F^2,
\end{align*}
hence $\|Y-Y'\|_F \le \frac{1}{2}\|S-S'\|_F$.
\end{proof}

\section{Proof of Proposition~\ref{prop:bandwise_attenuation}}
\label{proof:prop_bandwise_attenuation}

\begin{proof}
Because $f_d^{\mathcal{L}}(\tilde{L}_d)$ and $f_d^{\mathcal{H}}(\tilde{L}_d)$ are spectral filters of $\tilde{L}_d$,
they diagonalize in the eigenbasis of $\tilde{L}_d$:
\[
f_d^{\mathcal{L}}(\tilde{L}_d)=U_d f_d^{\mathcal{L}}(\tilde{\Lambda}_d)U_d^\top,
\qquad
f_d^{\mathcal{H}}(\tilde{L}_d)=U_d f_d^{\mathcal{H}}(\tilde{\Lambda}_d)U_d^\top.
\]
Hence,
\[
f_d(\tilde{L}_d)=f_d^{\mathcal{L}}(\tilde{L}_d)\,f_d^{\mathcal{H}}(\tilde{L}_d)
=
U_d\,\mathrm{diag}\!\Big(f_d^{\mathcal{L}}(\tilde{\lambda}_{d}^{(i)})\,f_d^{\mathcal{H}}(\tilde{\lambda}_{d}^{(i)})\Big)_{i=1}^N\,U_d^\top.
\]
Also, $P_{\Omega,d}=U_d\,\mathrm{diag}(\mathbf{1}_{i\in\Omega})\,U_d^\top$.
Let $\hat{x}=U_d^\top x$. Then:
\begin{align*}
\|P_{\Omega,d}\, f_d(\tilde{L}_d)\, x\|_2^2
&=
\sum_{i\in\Omega}\big|f_d^{\mathcal{L}}(\tilde{\lambda}_{d}^{(i)})\,f_d^{\mathcal{H}}(\tilde{\lambda}_{d}^{(i)})\big|^2\,|\hat{x}_i|^2 \\
&\le
\Big(\max_{i\in\Omega}\big|f_d^{\mathcal{L}}(\tilde{\lambda}_{d}^{(i)})\,f_d^{\mathcal{H}}(\tilde{\lambda}_{d}^{(i)})\big|^2\Big)
\sum_{i\in\Omega}|\hat{x}_i|^2 \\
&=
\Big(\max_{i\in\Omega}\big|f_d^{\mathcal{L}}(\tilde{\lambda}_{d}^{(i)})\,f_d^{\mathcal{H}}(\tilde{\lambda}_{d}^{(i)})\big|^2\Big)
\|P_{\Omega,d}\,x\|_2^2.
\end{align*}
Taking square roots yields the first inequality in Prop.~\ref{prop:bandwise_attenuation}.
The second inequality follows immediately from:
\[
\max_{i\in\Omega_{\mathrm{high}}}\big|f_d^{\mathcal{L}}(\tilde{\lambda}_{d}^{(i)})\,f_d^{\mathcal{H}}(\tilde{\lambda}_{d}^{(i)})\big|
\le
\Big(\max_{i\in\Omega_{\mathrm{high}}}|f_d^{\mathcal{L}}(\tilde{\lambda}_{d}^{(i)})|\Big)
\Big(\max_{i\in\Omega_{\mathrm{high}}}|f_d^{\mathcal{H}}(\tilde{\lambda}_{d}^{(i)})|\Big).
\]
\end{proof}


\section{Proof of Lemma~\ref{lemma:ce_lipschitz}}
\label{proof:lemma_ce_lipschitz}

\begin{proof}
Recall the softmax cross-entropy written in terms of logits $z\in\mathbb{R}^{C}$ and a one-hot label $y\in\{0,1\}^{C}$:
\[
\ell(z,y)=\log\Big(\sum_{c=1}^{C} e^{z_c}\Big)-\sum_{c=1}^{C} y_c z_c.
\]
Let $c^\star$ be the true class so that $y_{c^\star}=1$ and $y_c=0$ for $c\neq c^\star$. Then, we have:
\[
\ell(z,y)=\log\Big(\sum_{c=1}^{C} e^{z_c}\Big)-z_{c^\star}.
\]
Its gradient with respect to $z$ is:
\[
\nabla_z \ell(z,y)=\mathrm{softmax}(z)-y.
\]
Since $\mathrm{softmax}(z)$ is a probability vector in the simplex, we have $\|\mathrm{softmax}(z)\|_2\le 1$ and $\|y\|_2=1$, hence
\[
\|\nabla_z \ell(z,y)\|_2=\|\mathrm{softmax}(z)-y\|_2 \le \sqrt{2}.
\]
A bounded gradient implies Lipschitzness: for all $z,z'\in\mathbb{R}^{C}$,
\[
|\ell(z,y)-\ell(z',y)| \le \sup_{\xi}\|\nabla \ell(\xi,y)\|_2\,\|z-z'\|_2 \le \sqrt{2}\,\|z-z'\|_2.
\]
This proves the $\sqrt{2}$-Lipschitz claim.

For boundedness, if $\|z\|_{\infty}\le B_z$ then $\max_c z_c \le B_z$ and $\min_c z_c\ge -B_z$. Therefore, we obtain:
\[
\log\Big(\sum_{c=1}^{C} e^{z_c}\Big)\le \log\big(C\,e^{B_z}\big)=\log(C)+B_z,
\qquad
-z_{c^\star}\le B_z,
\]
which yields $\ell(z,y)\le \log(C)+2B_z$.
\end{proof}

\section{Proof of Proposition~\ref{prop:rad_scaling}}
\label{proof:prop_rad_scaling}

\begin{proof}
Fix a labeled sample $S=\{(x_i,y_i)\}_{i=1}^{n}$ and let $\Sigma\in\{-1,+1\}^{n\times C}$ be a matrix of i.i.d.\ Rademacher signs. For any matrix $Z\in\mathbb{R}^{n\times C}$, denote the Frobenius inner product by $\langle \Sigma, Z\rangle := \mathrm{tr}(\Sigma^\top Z)$.

For dimension $d$, the score matrix on the sample can be written as
\[
Z_d = A_d\, Z_0\, H_d,
\]
where $Z_0\in\mathbb{R}^{n\times C}$ collects the MLP outputs on the sample and $A_d$ is the (sample-restricted) linear propagation operator induced by $\hat{L}_d$. Since restricting to a subset of rows/columns can not increase the spectral norm, we have $\|A_d\|_2\le \|\hat{L}_d\|_2$.

By definition of empirical Rademacher complexity,
\begin{align*}
\mathfrak{R}_{n}(\mathcal{F}_d)
&=
\frac{1}{n}\,\mathbb{E}_{\Sigma}\Big[\sup_{Z_0\in\mathcal{F}_0(S)} \langle \Sigma,\;A_d Z_0 H_d\rangle\Big]\\
&=
\frac{1}{n}\,\mathbb{E}_{\Sigma}\Big[\sup_{Z_0\in\mathcal{F}_0(S)} \mathrm{tr}(\Sigma^\top A_d Z_0 H_d)\Big]\\
&=
\frac{1}{n}\,\mathbb{E}_{\Sigma}\Big[\sup_{Z_0\in\mathcal{F}_0(S)} \mathrm{tr}\big((A_d^\top \Sigma H_d^\top)^\top Z_0\big)\Big]\\
&=
\frac{1}{n}\,\mathbb{E}_{\Sigma}\Big[\sup_{Z_0\in\mathcal{F}_0(S)} \langle A_d^\top \Sigma H_d^\top,\; Z_0\rangle\Big].
\end{align*}
Using Cauchy--Schwarz for the Frobenius inner product,
\[
\langle A_d^\top \Sigma H_d^\top,\; Z_0\rangle
\le
\|A_d^\top \Sigma H_d^\top\|_F\,\|Z_0\|_F.
\]
Moreover, by submultiplicativity of norms and $\|M\|_F\le \sqrt{\mathrm{rank}(M)}\,\|M\|_2$, we can bound the transformation of Rademacher signs as:
\[
\|A_d^\top \Sigma H_d^\top\|_F
\le
\|A_d^\top\|_2\,\|\Sigma\|_F\,\|H_d^\top\|_2
=
\|A_d\|_2\,\|H_d\|_2\,\|\Sigma\|_F.
\]
This shows that the supremum over the transformed class is at most scaled by $\|A_d\|_2\|H_d\|_2$. Since $\|A_d\|_2\le \|\hat{L}_d\|_2$, we obtain:
\[
\mathfrak{R}_{n}(\mathcal{F}_d)
\le
\|\hat{L}_d\|_2\,\|H_d\|_2\,\mathfrak{R}_{n}(\mathcal{F}_0),
\]
which is exactly Eq.~(\ref{eq:rad_scaling_main}).
\end{proof}

\section{Proof of Theorem~\ref{thm:generalization}}
\label{proof:thm_generalization}

\begin{proof}
Let $\mathcal{G}$ denote the class of dimension-averaged predictors induced by \methodname, and consider the loss $\ell$ in Eq.~(\ref{eq:ce_loss_logits_form}). Under the i.i.d.\ sampling assumption, a standard Rademacher generalization result for bounded losses implies that with probability at least $1-\delta$,
\[
\mathcal{R}\le \widehat{\mathcal{R}} + 2\,\mathfrak{R}_{n}(\ell\circ \mathcal{G})
+ 3\,B_{\max}\sqrt{\frac{\log(2/\delta)}{2n}},
\]
where $\mathfrak{R}_{n}(\ell\circ \mathcal{G})$ is the empirical Rademacher complexity of the loss-composed class and $B_{\max}$ upper bounds the loss range.

By Lemma~\ref{lemma:ce_lipschitz}, $\ell(\cdot,y)$ is $\sqrt{2}$-Lipschitz in its logits argument with respect to $\|\cdot\|_2$. Therefore, by the vector contraction principle, we have:
\[
\mathfrak{R}_{n}(\ell\circ \mathcal{G})
\le
\frac{\sqrt{2}}{D}\sum_{d=1}^{D}\mathfrak{R}_{n}(\mathcal{F}_d).
\]
Finally, applying Proposition~\ref{prop:rad_scaling} to each $\mathfrak{R}_{n}(\mathcal{F}_d)$ yields
\[
\mathfrak{R}_{n}(\ell\circ \mathcal{G})
\le
\frac{\sqrt{2}}{D}\Big(\sum_{d=1}^{D}\|\hat{L}_d\|_2\,\|H_d\|_2\Big)\,\mathfrak{R}_{n}(\mathcal{F}_0).
\]
Substituting this inequality into the generalization bound above proves Eq.~(\ref{eq:gen_bound_main}).
\end{proof}

\end{document}